\documentclass[12pt]{article}

\RequirePackage{amsthm,amsmath,amssymb,amsfonts,graphicx,bm,bbm,mathrsfs}
\usepackage{authblk}
\usepackage{enumerate}
\usepackage[margin = 1in]{geometry}
\usepackage{parskip}
\usepackage{fancyhdr}
\usepackage[usestackEOL]{stackengine}

\usepackage{times}
\usepackage{natbib}
\usepackage{float}
\usepackage{pgfplots}
\usepgfplotslibrary{statistics}

\usepackage{fullpage}
\usepackage[OT1]{fontenc}
\usepackage[utf8]{inputenc}
\RequirePackage[colorlinks,citecolor=blue,urlcolor=blue,breaklinks=true]{hyperref}
\usepackage{breakcites}
\usepackage{cleveref, autonum}
\usepackage{booktabs} 
\usepackage{tikz}
\usepackage[title]{appendix}
\usepackage{url}
\usepackage{comment}
\usepackage{xcolor}

\theoremstyle{plain}
\newtheorem{theorem}{Theorem}
\newtheorem{lemma}[theorem]{Lemma}

\newtheorem{corollary}[theorem]{Corollary}
\newtheorem{definition}[theorem]{Definition}
\newtheorem{assumption}{Assumption}
\renewcommand{\theassumption}{A}

\newcommand{\bI}{\bm{I}}

\newcommand\bx{\bm{x}}
\newcommand\by{\bm{y}}

\newcommand\bw{\bm{w}}

\newcommand{\PP}{\mathbb{P}}
\newcommand{\EE}{\mathbb{E}}
\newcommand{\bbR}{\mathbb{R}}
\newcommand{\bbH}{\mathbb{H}}
\newcommand{\mX}{\mathcal{X}}
\newcommand{\tr}{\mathrm{tr}}
\newcommand{\s}{\sum_{i=1}^{\infty}}
\newcommand{\Ltwo}{L^2_{p_X}(\mX)}

\newcommand{\bounda}{$\tilde{\bbH}$-bound}
\newcommand{\boundb}{${\bbH}$-bound}

\begin{document}

\title{On the Estimation of Derivatives Using Plug-in Kernel Ridge Regression Estimators}
\author[ ]{Zejian Liu\thanks{zejian.liu@rice.edu}}
\author[ ]{Meng Li \thanks{meng@rice.edu}}
\affil[ ]{Department of Statistics, Rice University}
\date{}

\maketitle

\begin{abstract}
We study the problem of estimating the derivatives of a regression function, which has a wide range of applications as a key nonparametric functional of unknown functions. Standard analysis may be tailored to specific derivative orders, and parameter tuning remains a daunting challenge particularly for high-order derivatives. In this article, we propose a simple plug-in kernel ridge regression (KRR) estimator in nonparametric regression with random design that is broadly applicable for multi-dimensional support and arbitrary mixed-partial derivatives. We provide a non-asymptotic analysis to study the behavior of the proposed estimator in a unified manner that encompasses the regression function and its derivatives, leading to two error bounds for a general class of kernels under the strong $L_\infty$ norm. In a concrete example specialized to kernels with polynomially decaying eigenvalues, the proposed estimator recovers the minimax optimal rate up to a logarithmic factor for estimating derivatives of functions in H\"older and Sobolev classes. Interestingly, the proposed estimator achieves the optimal rate of convergence with the same choice of tuning parameter for any order of derivatives. Hence, the proposed estimator enjoys a \textit{plug-in property} for derivatives in that it automatically adapts to the order of derivatives to be estimated, enabling easy tuning in practice. Our simulation studies show favorable finite sample performance of the proposed method relative to several existing methods and corroborate the theoretical findings on its minimax optimality.  
\end{abstract}

\section{Introduction}
\label{submission}

Estimating the derivatives of the regression function has a wide range of applications in many areas, such as cosmology \citep{holsclaw2013gaussian}, spatial process models \citep{banerjee2003directional}, and shape-constrained function estimation that builds on the derivative process or virtual derivative observations~\citep{riihimaki2010gaussian,wang2016estimating}. Furthermore, derivative estimation may improve the computational efficiency for nonlinear dynamic system identification \citep{solak2003derivative}, while serving as a vital tool in detecting local extrema \citep{song2006nonparametric,li2021semiparametric} and efficient modeling of functional data \citep{dai2018derivative}.

Existing methods for estimating derivatives of regression functions include smoothing spline, local polynomial regression, and difference-based methods. Local polynomial regression and smoothing spline base the estimation of derivatives on estimating the regression function. Key smoothing parameters in these two methods often depend on the order of the derivative being estimated and are difficult to choose in practice \citep{wahba1990optimal,charnigo2011generalized}. Difference-based methods require boundary correction, hindering theoretical studies such as those establishing uniform convergence rates. Moreover, existing methods are either restricted to the fixed design setting, or only applicable to one-dimensional support and low-order derivatives. The goal of this article is to develop a unified framework for derivative estimation that achieves broadened applicability and enables simple optimal parameter tuning with theoretical guarantees.

In this paper, we propose a simple plug-in kernel ridge regression (KRR) estimator for the derivatives of the regression function and develop a non-asymptotic framework that provides theoretical support for general kernels. We consider a random design setting with multi-dimensional support, and derive convergence rates for partial mixed derivatives of arbitrary order under the strong $L_\infty$ norm. In a concrete example where the regression function belongs to a H\"older or Sobolev class, we show that the proposed estimator is nearly minimax optimal with the same choice of tuning parameters for \textit{any} order of derivatives to be estimated. Hence, unlike methods such as smoothing spline, the proposed estimator remarkably adapts to the order of derivatives and achieves the so-called \textit{plug-in property} \citep{bickel2003nonparametric} for derivative estimation. This leads to immediate insight for parameter tuning, which along with the closed-form expression substantially facilitates the implementation of the proposed method in broad settings.

Kernel ridge regression \citep{wahba1990spline,gyorfi2006distribution,cucker2007learning}, also known as regularized least squares, is a popular technique in supervised learning and has been widely used in an immense variety of areas, including computer vision \citep{cheng2016improved}, speech recognition \citep{chen2016efficient}, forecasting \citep{exterkate2016nonlinear}, and biomedical fields \citep{mohapatra2016microarray}.

There has been a rich literature on the theoretical guarantees of KRR \citep{cucker2002best,zhang2005learning,caponnetto2007optimal,steinwart2009optimal,mendelson2010regularization}. However, theory on nonparametric functionals of KRR estimators such as derivatives is comparatively underdeveloped. We contribute to the growing literature of KRR by developing non-asymptotic analysis for derivatives of arbitrary order, with added focus on its generality to encompass a large class of kernels and the strong $L_\infty$ norm.

\subsection{Related work}
One popular method to estimate function derivatives is to differentiate estimates of the regression functions. For example, smoothing spline produces derivative estimation by differentiating the spline basis. \cite{stone1985additive,zhou2000derivative} studied theoretical properties of smoothing spline, including the $L_2$ minimax optimal convergence rate. Local polynomial regression is another standard method, which relies on local polynomial fitting obtained by Taylor expansion; \cite{fan1996local,delecroix1996nonparametric} provided asymptotic normality and strong uniform consistency for local polynomial regression, respectively. However, the smoothing parameter in these methods typically depends on the order of the derivative and is usually difficult to choose in practice \citep{wahba1990optimal,wang2015derivative}. The implementation of these methods is also specific to one particular derivative order. {\cite{yatracos1989estimation} related error bounds for general plug-in derivative estimators to those for the original estimators, concluding that derivative estimation is more challenging than the estimation of the original function. This line of research was recently expanded by \cite{yatracos2019plug}, focusing on mixing density estimation instead of nonparametric regression.

Difference-based methods have attracted increasing attention. These methods create a new noisy dataset with derivatives as the mean, followed by nonparametric smoothing \citep{muller1987bandwidth,hardle1990applied}. Along this line, \cite{charnigo2011generalized,de2013derivative} proposed an empirical derivative estimator with improved variance and established pointwise consistency. \cite{wang2015derivative} derived an asymptotic $L_2$ convergence rate for estimating the first derivative. However, they required the true regression function to be five times differentiable, which is a very strict assumption. \cite{liu2018derivative,liu2020smoothed} extended the difference-based estimator to random design. \cite{wang2019robust} adopted $L_1$ regression instead of least squares regression, improving the robustness to outliers and heavy-tailed errors. However, difference-based methods typically aim at estimating the first or second derivative of regression functions with one-dimensional support, and have limited developments for high-order derivatives or multi-dimensional cases. Moreover, difference-based estimators require boundary correction in general, necessitating a separate treatment when studying their behavior at the boundaries.

In this article, we provide a non-asymptotic analysis for the proposed plug-in KRR estimator. In the literature of KRR for nonparametric regression, \cite{cucker2002best} provided a non-asymptotic upper bound under the $L_2$ norm, utilizing the covering number of an open subset of the reproducing kernel Hilbert space (RKHS). \cite{smale2005shannon,smale2007learning} replaced the covering number technique by the method of integral operators and obtained tighter bounds, but assumed the outputs to be uniformly bounded above, excluding the Gaussian error. This assumption was later relaxed by moment conditions \citep{wang2011optimal,guo2013concentration}. However, they did not particularly focus on learning rates for derivatives of KRR estimators. The optimality of the induced learning rates and parameter tuning that is of great practical relevance, in the presence of varying derivative orders, have not been studied in the literature. Note that derivatives with kernel methods have been considered in different settings, including nonparametric sparse regression \citep{rosasco2013nonparametric} and semi-supervised learning \citep{cabannes2021overcoming}.

\subsection{Contributions}
Our contributions can be summarized as follows: 
\begin{enumerate}[(1)]
\item We propose a plug-in KRR estimator for derivatives of arbitrary order. The proposed estimator is analytically given and applicable for multi-dimensional support and sub-Gaussian error, enabling fast computation and broad practicability. We allow the derivative order to be zero throughout the article and thus unify the study of the regression function and its derivatives.

\item We provide two non-asymptotic error bounds for the proposed plug-in KRR estimators under the $L_\infty$ norm: the \bounda{} for Mercer kernels with uniformly bounded eigenfunctions (Section~\ref{sec:equivalent}) and the \boundb{} for all Mercer kernels (Section~\ref{sec:general.mercer}). To the best of our knowledge, learning rates for derivatives of KRR estimators have not been addressed in the literature. Our analysis rests on an operator-theoretic approach, equivalent kernels, and the Hanson-Wright inequality, encompassing a general class of kernel functions and the strong $L_\infty$ norm.

\item In a concrete example where the kernel has polynomially decaying eigenvalues and the regression function belongs to a H\"older or Sobolev class, we show that our general analysis recovers the nearly minimax optimal $L_2$ convergence rate, suggesting the sharpness of the established bounds (Section~\ref{sec:matern}). Given the smoothness level of the regression function, the rate-optimal estimation is achieved under the same choice of the regularization parameter that does not depend on the derivative order. Therefore, the proposed estimator enjoys a remarkable plug-in property that it automatically adapts to the order of the derivative to be estimated, leading to easy tuning in practice.
\end{enumerate}

\subsection{Notation}
Let $\mathbb{N}$ be the set of all positive integers and write $\mathbb{N}_0=\mathbb{N}\cup\{0\}$. We let $C(\mX)$ denote the space of continuous functions. For a multi-index $\bm\beta=(\beta_1,\ldots,\beta_d)\in\mathbb{N}_0^d$, we write $|\bm\beta|=\beta_1+\cdots+\beta_d$ and $\partial^{\bm\beta}=\partial^{\beta_1}_{x_1}\cdots\partial^{\beta_d}_{x_d}$. For any $m\in\mathbb{N}$, let $C^m(\mX)$ stand for the space of all functions possessing continuous mixed partial derivatives up to order $m$, \textit{i.e.}, $C^m(\mX)=\{f:\mX\rightarrow\bbR |  \partial^{\bm\beta} f\in C(\mX) \ {\rm for\ all\ } \bm\beta\in\mathbb{N}_0^d \ {\rm with\ } |\bm\beta|\leq m\}$. Let $C(\mX,\mX)$ denote the space of continuous bivariate functions and $C^{2m}(\mX,\mX)=\{K:\mX\times \mX\rightarrow\bbR |  \partial^{\bm\beta,\bm\beta} K\in C(\mX,\mX) \ {\rm for\ all\ } \bm\beta\in\mathbb{N}_0^d \ {\rm with\ } |\bm\beta|\leq m\}$ denote the space of $m$-times continuously differentiable bivariate functions, where $\partial^{\bm\beta,\bm\beta}K(\bx,\bx')=\partial^{\bm\beta}_{\bx}\partial^{\bm\beta}_{\bx'}K(\bx,\bx')$. For any $f: \mX \rightarrow \bbR$, let $\|f\|_\infty$ be the $L_\infty$ norm. For two sequences $a_n$ and $b_n$, we write $a_n \lesssim b_n$ if $a_n \leq C b_n$ for a universal constant $C>0$, and $a_n \asymp b_n$ if $a_n \lesssim b_n$ and $b_n \lesssim a_n$.

\section{Plug-in KRR estimator for function derivatives} \label{sec:model}
Suppose that we have $n$ iid observations $\{X_i,y_i\}_{i=1}^{n}$ from an unknown data generating probability $\PP_0$ on $\mX\times \bbR$, where $\mX\subset \bbR^d$ is a compact metric space for $d\geq 1$. Denote the marginal distribution on $\mX$ by $\PP_X$ with Lebesgue density $p_X$. Let $\Ltwo$ be the $L_2$ space with respect to the measure $\PP_X$, with the $L_2$ norm $\|f\|_2=(\int_\mX f^2d\PP_X)^{1/2}$ and the inner product $\left<f,g\right>_2=(\int_\mX fgd\PP_X)^{1/2}$. The regression model is given by
\begin{equation} \label{eq:model} 
y_i=f_0(X_i)+\varepsilon_i,
\end{equation}
where the random error $\varepsilon_i$ is sub-Gaussian with mean zero and variance proxy $\sigma^2$, \textit{i.e.}, $\EE[\varepsilon_i]=0$ and $\EE[e^{t\varepsilon_i}]\leq e^{\sigma^2t^2/2}$ for any $t \in \mathbb{R}$. Given a multi-index $\bm\beta=(\beta_1,\ldots,\beta_d)\in\mathbb{N}_0^d$, our goal is to estimate $\partial^{\bm\beta}f_0$, the mixed partial derivative of the regression function, assuming its existence.

Let $X=(X_1^T,\ldots,X_n^T)^T\in\bbR^{n\times d}$ and $\by=(y_1,\ldots,y_n)^T\in \bbR^n$. Let $K(\cdot,\cdot):\mX\times\mX\rightarrow \bbR$ be a Mercer kernel, \textit{i.e.}, a continuous, symmetric, and positive definite bivariate function. In this article, we propose the following closed-form plug-in KRR estimator for $\partial^{\bm\beta}f_0$, assuming differentiability of $K$:
\begin{equation}\label{eq:KRR.multi}
\widehat{\partial^{\bm\beta}f_0}(\bx) =: \partial^{\bm\beta}\hat{f}_n(\bx) = [\partial^{\bm\beta}K(\bx, X)] [K(X, X) + n \lambda \bI_n]^{-1} \by,
\end{equation}
where $\partial^{\bm\beta}K(\bx, X)=(\partial_{x_i}^{\beta_i}K(x_i,X_j))_{1\leq i\leq d, 1\leq j\leq n}$ is a $d$ by $n$ matrix, $K(X, X)$ is the $n$ by $n$ matrix $(K(X_i, X_j))_{1\leq i, j \leq n}$, and $\lambda>0$ is a regularization parameter that possibly depends on the sample size $n$. Here $\hat{f}_n(\bx) = \widehat{\partial^{\bm 0}f_0}(\bx)$ is the classical KRR estimator for the regression function $f_0$. It is well known that $\hat{f}_n$ is also the solution to the following optimization problem:
\begin{equation}\label{eq:KRR} 
\hat{f}_n = \underset{f \in \bbH}{\arg\min} \left\{\frac{1}{n}\sum_{i = 1}^n (y_i - f(X_i))^2 + \lambda \|f\|^2_{\bbH}\right\},
\end{equation}
where $(\bbH,\|\cdot\|_\bbH)$ is the reproducing kernel Hilbert space (RKHS) induced by the kernel $K$.

The closed-form expression in~\eqref{eq:KRR.multi} enables fast calculation for any order of derivatives. The proposed estimator is applicable for $d$-dimensional support with $d \geq 1$. Taking one-dimensional support $\mX\subset \bbR$ as a special case, the plug-in KRR estimator for $f_0^{(m)}$ with $m\in\mathbb{N}_0$ is
\begin{equation}\label{eq:KRR.uni}
\hat{f}_n^{(m)}(x) = [\partial^mK(x, X)] [K(X, X) + n \lambda \bI_n]^{-1} \by,
\end{equation}
where $\partial^mK(x, X)=(\partial
^m_xK(x,X_j))_{1\leq j\leq n}$ is a 1 by $n$ vector. In addition, \eqref{eq:KRR.multi} enables convenient inference. For example, the proposed estimator at any $\bx$ is normally distributed for Gaussian error $\varepsilon_i\sim N(0,\sigma^2)$ with variance given by
\begin{equation}
\sigma^2[\partial^{\bm\beta}K(\bx, X)] [K(X, X) + n \lambda \bI_n]^{-2} [\partial^{\bm\beta}K(\bx, X)]^T.
\end{equation}

\section{Non-asymptotic analysis}\label{sec:nonasymptotic}
We take an operator-theoretic approach to study non-asymptotic properties of the plug-in KRR estimator, which characterizes behaviors of the proposed estimator, provides insights for choosing regularization parameters, and leads to asymptotic optimality. 

\subsection{Preliminaries}\label{sec:preliminaries}
We begin with reviewing preliminaries and introducing commonly used notation \citep{smale2005shannon,smale2007learning,steinwart2009optimal}. For any $f\in\Ltwo$, we introduce an integral operator $L_K: \Ltwo \rightarrow \bbH$ $\subset \Ltwo$ defined by 
\begin{equation} \label{eq:int.operator}
L_K(f)(\bx) = \int_{\mX} K(\bx, \bx') f(\bx') d\PP_X(\bx'), \quad \bx \in \mX. 
\end{equation}
Since $L_K$ is compact, positive definite, and self-adjoint on $\Ltwo$ (\textit{i.e.}, as an operator mapping $\Ltwo$ to $\Ltwo$), by the spectral theorem (see, \textit{e.g.}, Theorem A.5.13 in \cite{steinwart2008support}), there exist countable pairs of eigenvalues and eigenfunctions $(\mu_i,\psi_i)_{i\in\mathbb{N}}\subset (0,\infty)\times\Ltwo$ of $L_K$ such that
\begin{equation}
L_K\psi_i=\mu_i\psi_i,\quad i\in \mathbb{N},
\end{equation}
where $\{\psi_i\}_{i=1}^{\infty}$ form an orthonormal basis of $\Ltwo$ and $\mu_1\geq\mu_2\geq \cdots> 0$ with $\lim\limits_{i\rightarrow\infty}\mu_i=0$. By Mercer's Theorem, we have that for any $\bx,\bx'\in\mX$,
\begin{equation}
K(\bx,\bx')=\sum_{i=1}^{\infty}\mu_i\psi_i(\bx)\psi_i(\bx'),
\end{equation}
where the convergence is absolute and uniform. It follows that $\bbH$ can be characterized by a series representation 
\begin{equation}\label{eq:RKHS.norm}
\bbH=\left\{f\in \Ltwo: \|f\|^2_{\bbH}=\sum_{i=1}^{\infty}\frac{f_i^2}{\mu_i}<\infty\right\},
\end{equation}
where $f_i=\left<f,\psi_i\right>_{2}$. The corresponding inner product is given by $\left<f,g\right>_{\bbH}=\sum_{i=1}^{\infty}{f_ig_i}/{\mu_i}$ for any $f=\sum_{i=1}^{\infty}f_i\psi_i$ and $g=\sum_{i=1}^{\infty}g_i\psi_i$ in $\bbH$.

We then define the sample analog $L_{K, X}: \bbH \rightarrow \bbH$ by 
\begin{equation}\label{eq:sample.analog}
L_{K, X}(f) = \frac{1}{n} \sum_{i = 1}^n f(X_i) K_{X_i}, 
\end{equation}
where $K_{\bx}(\cdot) := K(\bx, \cdot)$.
It is easy to see $L_{K, X}$ is also a compact, positive definite, self-adjoint operator because for any $f, g \in \bbH$, we have 
\begin{equation}
\langle f, L_{K, X}g \rangle_{\bbH} = \frac{1}{n} \sum_{i = 1}^n f(X_i) g(X_i) = \langle L_{K, X}f, g \rangle_{\bbH}
\end{equation}
and $\langle f, L_{K, X}f \rangle_{\bbH} \geq 0$. Thus, the eigenvalues of $L_{K, X}$ are all non-negative, which implies  
\begin{equation}\label{eq:dummy1}
\|(L_{K, X} + \lambda I)^{-1} f \|_{\bbH} \leq \frac{1}{\lambda} \|f\|_{\bbH},
\end{equation}
for any $f \in \bbH$. We remark that the operator $L_K$ can also be defined on $\bbH$ and so does $L_{K, X}$ on the space of all bounded functions; we use the same notation when they are defined on different domains. 

We further consider a proximate function of $f_0$ in $\bbH$ 
\begin{equation}\label{eq:flambda}
f_{\lambda} = (L_K + \lambda I)^{-1}L_K f_{0},
\end{equation}
where $I$ is the identity operator. The function $f_{\lambda}$ is chosen this way to minimize the population counterpart of~\eqref{eq:KRR}, \textit{i.e.}, 
\begin{equation} \label{eq:KRR.population}
f_{\lambda} = \underset{f \in \bbH}{\arg\min} \; \left\{\|f - f_{0}\|_{2}^2 + \lambda \|f\|^2_{\bbH} \right\}.
\end{equation}

We next present two non-asymptotic bounds for Mercer kernels with uniformly bounded eigenfunctions and general Mercer kernels, respectively, which provide the basis for more specific rate calculation and may be of independent interest for the KRR community.

\subsection{\bounda{} for Mercer kernels with uniformly bounded eigenfunctions}\label{sec:equivalent}
The proximate function $f_\lambda$ can be obtained using another integral operator $L_{\tilde{K}}$ through $f_\lambda = L_{\tilde{K}} f_0$, where $\tilde{K}$ is the so-called equivalent kernel \citep{rasmussen2006,sollich2005using}. Compared to $K$, the equivalent kernel $\tilde{K}$ has the same eigenfunctions but its eigenvalues are altered to $\nu_i={\mu_i}/{(\lambda+\mu_i)}$ for $i\in\mathbb{N}$, \textit{i.e.}, 
\begin{equation}
\tilde{K}(\bx,\bx')=\s \nu_i\psi_i(\bx)\psi_i(\bx').
\end{equation}
Let $\tilde{\bbH}$ be the RKHS induced by $\tilde{K}$, which is equivalent to $\bbH$ as a functional space, but with a different inner product
\begin{equation}
\left<f,g\right>_{\tilde{\bbH}}=\left<f,g\right>_2+\lambda\left<f,g\right>_\bbH.
\end{equation}
Let the corresponding RKHS norm be $\|\cdot\|_{\tilde{\bbH}}$. Note that $\tilde{K}$ is also a Mercer kernel; thus, all preliminaries in Section~\ref{sec:preliminaries} hold for $\tilde{K}$. For example, in view of~\eqref{eq:sample.analog}, we can similarly define the sample analog by
\begin{equation}
L_{\tilde{K}, X}(f) = \frac{1}{n} \sum_{i = 1}^n f(X_i) \tilde{K}_{X_i}, 
\end{equation}
which is compact, positive definite, and self-adjoint. Note that throughout the article, the tilde notation such as $\tilde{K}$ and $\tilde{\bbH}$ indicates dependence on $\lambda$, although $\lambda$ may not be explicitly spelled out in the notation with exceptions of $\tilde{\kappa}^2_\lambda$ and $\tilde{\kappa}_{\bm\beta,\lambda}^2$ defined later.

We introduce the following assumption on the differentiability of the regression function and the covariance kernel.

\renewcommand{\theassumption}{A}
\begin{assumption}\label{ass:A}
There exists an $m\in\mathbb{N}_0$ such that $f_0 \in C^m(\mX)$ and $K\in C^{2m}(\mX,\mX)$.
\end{assumption}

Under Assumption~\ref{ass:A}, we can estimate $\partial^{\bm\beta} f_0$ by $\partial^{\bm\beta} \hat f_n$ for any $|\bm\beta| \leq m$. Note that we do not necessarily require that $f_0$ and $K$ have exactly the same smoothness level. Indeed, if there is a mismatch between the differentiability of $f_0$ and $K$ such that $f_0 \in C^{m_1}(\mX)$ and $K \in C^{2m_2}(\mX,\mX)$, we can take $m = \min\{m_1, m_2\}$ to satisfy Assumption~\ref{ass:A}. From the perspective of kernel selection, this assumption indicates that for estimating $\partial^{\bm\beta} f_0$, we should choose a kernel $K\in C^{2m}(\mX,\mX)$ such that $m \geq |\bm\beta|$. Such kernels are widely available. For example, it is satisfied by a general class of kernels with polynomially decaying eigenvalues (given in Definition~\ref{def:kalpha}) under mild conditions. In particular, the Mat\'ern kernel is known to be $2m$ times differentiable if and only if the smoothness parameter $\nu > m$ \citep[Chapter 2.7]{stein1999interpolation}. The squared exponential kernel as a limit case of Mat\'ern kernel satisfies Assumption~\ref{ass:A} for any $m$. Assumption~\ref{ass:A} directly implies that $\psi_i \in C^m(\mX)$.

Define $\tilde{\kappa}^2_\lambda:={\sup_{\bx\in\mX}\tilde{K}(\bx,\bx)}$. It is easy to see $\tilde{\kappa}^2_\lambda\leq C_\psi^2\sum_{i=1}^{\infty}\nu_i\lesssim\sum_{i=1}^{\infty}{\mu_i}/{(\lambda+\mu_i)},$ where the last expression is the effective dimension \citep{zhang2005learning} of the kernel $K$ with respect to $\Ltwo$. We also define high-order analogies of $\tilde{\kappa}^2_\lambda$ for any multi-index $\bm\beta\in \mathbb{N}_0^d$ and $|\bm\beta| \leq m$:
\begin{equation}\label{eq:high.order.kappa}
\tilde{\kappa}_{\bm\beta,\lambda}^2:=\sup_{\bx\in\mX}\partial^{\bm\beta,\bm\beta}\tilde{K}(\bx,\bx)=\sup_{\bx\in \mX}\s\frac{\mu_i}{\lambda+\mu_i}\{\partial^{\bm\beta}\psi_i(\bx)\}^2.
\end{equation}
Note that $\tilde{\kappa}_\lambda^2 = \tilde{\kappa}_{\bm\beta,\lambda}^2$ with $\bm \beta = \bm 0$. In general, $\tilde{\kappa}_{\bm\beta,\lambda}^2$ is determined by the decay rate of the eigenvalues of $K$, the derivatives of the eigenfunctions, and the regularization parameter $\lambda$. For given $\lambda > 0$, there holds 
\begin{equation}
\tilde{\kappa}_{\bm\beta,\lambda}^2 \leq \lambda^{-1} \sup_{\bx\in \mX}\s{\mu_i}\{\partial^{\bm\beta}\psi_i(\bx)\}^2 = \lambda^{-1} \sup_{\bx\in\mX}\partial^{\bm\beta,\bm\beta}K(\bx,\bx) < \infty,
\end{equation}
where the boundedness of $\sup_{\bx\in\mX}\partial^{\bm\beta,\bm\beta}K(\bx,\bx)$ in the last step is due to Assumption~\ref{ass:A} as $\partial^{\bm\beta,\bm\beta}K(\bx,\bx)$ is a continuous bivariate function on a compact support $\mX$.

The following assumption on the eigenfunctions pertains to the equivalent kernel technique considered in this section; the error bound established in Section~\ref{sec:general.mercer} does not require such an assumption. 

\renewcommand{\theassumption}{B}
\begin{assumption}\label{ass:B}
There exists a constant $C_\psi>0$ such that $\|\psi_i\|_\infty\leq C_\psi$ for all $i\in\mathbb{N}$.
\end{assumption}

The following lemma indicates that functions in the RKHS inherit the differentiability of the kernel, and the $L_\infty$ norm of the derivative is upper bounded by the RKHS norm of the function.  
\begin{lemma}\label{lem:RKHS.derivative}
Under Assumption~\ref{ass:A}, $f\in C^m(\mX)$ for any $f\in \tilde{\bbH}$. Moreover, for any $\bm\beta\in\mathbb{N}_0^d$, $|\bm\beta|\leq m$, we have $\|\partial^{\bm\beta} f\|_\infty\leq\tilde\kappa_{\bm\beta, \lambda}\|f\|_{\tilde{\bbH}}$ \ for any $f\in \tilde{\bbH}$.
\end{lemma}

Theorem~\ref{thm:equivalent.bound} provides error bounds for Mercer kernels with bounded eigenfunctions.
\begin{theorem}[\bounda]\label{thm:equivalent.bound}
Under Assumptions~\ref{ass:A} and \ref{ass:B}, for any $\bm\beta\in\mathbb{N}_0^d$ with $|\bm\beta|\leq m$ and $\delta\in(0,1)$, it holds with probability at least $1-\delta$ that
\begin{align}
\|\partial^{\bm\beta}\hat{f}_n-\partial^{\bm\beta}f_0\|_\infty\leq &\ \|\partial^{\bm\beta}f_\lambda-\partial^{\bm\beta}f_0\|_\infty + \frac{\tilde{\kappa}_{\bm\beta,\lambda}\tilde{\kappa}_\lambda^{-1}C(n,\tilde{\kappa}_\lambda)}{1-C(n,\tilde{\kappa}_\lambda)}\|f_\lambda-f_0\|_\infty\\
&\ +\frac{1}{1-C(n,\tilde{\kappa}_\lambda)}\frac{C_1\tilde{\kappa}_{\bm\beta,\lambda}\tilde{\kappa}_\lambda\sigma\sqrt{\log (3/\delta)}}{\sqrt{n}},
\end{align}
where $C_1>0$ does not depend on $K$ or $n$, and $C(n,\tilde{\kappa}_\lambda)=\frac{\tilde{\kappa}^2_\lambda \sqrt{\log (3/\delta)}}{\sqrt{n}} \left(4 + \frac{4 \tilde{\kappa}_\lambda\sqrt{\log (3/\delta)}}{3 \sqrt{n}} \right)$.
\end{theorem}

\subsection{\boundb{} for general Mercer kernels}\label{sec:general.mercer}
The \bounda{} established in the preceding section relies on the crucial Assumption~\ref{ass:B} that the eigenfunctions are uniformly bounded, which does not necessarily hold for all Mercer kernels. For example, \cite{zhou2002covering} constructed a $C^\infty$ kernel that does not satisfy this assumption. To thoroughly study the learning rate in a more general setting, we provide another error bound under the RKHS norm $\|\cdot\|_{\bbH}$ for any Mercer kernel, referred to as the \boundb{}.

Let $f_{X, \lambda}$ be the noiseless counterpart of $\hat{f}_n$ by replacing noisy data with their means given by the true regression function, \textit{i.e.},
\begin{equation}
f_{X, \lambda} := K(\cdot, X)[K(X, X) + n \lambda \bI_n]^{-1} f_0(X),
\end{equation}
where $f_0(X) := (f_0(X_1), \ldots, f_0(X_n))^T$. An equivalent operator-based representation akin to \eqref{eq:sample.analog} gives $f_{X,\lambda}=(L_{K, X} + \lambda I)^{-1} L_{K, X} f_0$.

The following lemma is a parallel version of Lemma~\ref{lem:RKHS.derivative} but applies to $K$ and the $\|\cdot\|_{\bbH}$ norm. While $\tilde\kappa_{\bm{\beta},\lambda}^2$ relies on the regularization parameter $\lambda$, the characterization of using the $\|\cdot\|_{\bbH}$ norm only involves the single parameter $\kappa_{\bm\beta}^2$ of $K$, where $\kappa_{\bm\beta}^2:=\sup_{\bx\in\mX}\partial^{\bm\beta,\bm\beta} K(\bx,\bx) < \infty$ and $\kappa^2:=\kappa_{\bm 0}^2$.

\begin{lemma}\label{lem:H.derivative}
Under Assumption~\ref{ass:A}, $f\in C^m(\mX)$ for any $f\in\bbH$. Moreover, for any $\bm\beta\in\mathbb{N}_0^d$, $|\bm\beta|\leq m$, we have $\|\partial^{\bm\beta} f\|_\infty\leq\kappa_{\bm\beta}\|f\|_\bbH$ \ for any $f\in\bbH$.
\end{lemma}

Invoking Lemma~\ref{lem:H.derivative} and decomposing $\hat{f}_n-f_\lambda=(\hat{f}_n-f_{X,\lambda})+(f_{X,\lambda}-f_\lambda)$ yield convergence rates of the mixed partial derivatives of $\hat{f}_n$ under the $L_\infty$ norm.
\begin{theorem}[\boundb]\label{thm:mixed.partial}
Under Assumption~\ref{ass:A}, for any $\bm\beta\in\mathbb{N}_0^d$ with $|\bm\beta|\leq m$ and $\delta \in (0, 1)$, it holds with probability at least $1 - \delta$  that
\begin{align}
\label{eq:mixed.partial}
\|\partial^{\bm\beta}\hat{f}_n - \partial^{\bm\beta} f_0\|_{\infty}\leq &\ \|\partial^{\bm\beta}f_\lambda - \partial^{\bm\beta} f_0\|_{\infty}+\frac{\kappa_{\bm\beta}\kappa \|f_0\|_\infty \sqrt{\log(9/\delta)}}{\sqrt{n}\lambda} \left(10 + \frac{4 \kappa\sqrt{\log(9/\delta)}}{3 \sqrt{n\lambda}} \right)\\
&\ +\frac{C_2\kappa_{\bm\beta}\kappa \sigma \sqrt{\log(3/\delta)}}{\sqrt{n} \lambda},
\end{align}
where $C_2>0$ does not depend on $K$ or $n$.
\end{theorem}

We have established two non-asymptotic error bounds in Theorem~\ref{thm:equivalent.bound} and Theorem~\ref{thm:mixed.partial} under the $L_\infty$ norm. A few remarks are in order: 

{\bf Remark} Both \bounda{} and \boundb{} have three terms. This three-term structure stems from applying the triangle inequality to the decomposition $\partial^{\bm\beta}\hat{f}_n - \partial^{\bm\beta} f_0 =  (\partial^{\bm\beta}\hat{f}_{\lambda} - \partial^{\bm\beta} f_0) + (\partial^{\bm\beta}\hat{f}_n - \partial^{\bm\beta} f_{\lambda})$, with  $\|\partial^{\bm\beta}f_{\lambda}-\partial^{\bm\beta}f_0\|_{\infty}$ being the first term, and $\partial^{\bm\beta}\hat{f}_n - \partial^{\bm\beta} f_{\lambda}$ bounded by the second and third terms combined. \bounda{} and \boundb{} use different approaches to bound $\partial^{\bm\beta}\hat{f}_n - \partial^{\bm\beta} f_{\lambda}$. As the name suggests, \bounda{} employs the $\|\cdot\|_{\tilde{\bbH}}$ norm of $\partial^{\bm\beta}\hat{f}_n - \partial^{\bm\beta} f_{\lambda}$, while \boundb{} uses the $\|\cdot\|_{\bbH}$ norm. When the observations are noiseless, \textit{i.e.}, $\by=f_0(X)$ in \eqref{eq:KRR}, the two bounds can be simplified by letting $\sigma=0$, zeroing out the third term in both bounds. Indeed, all subsequent error bounds imply a noise-free version by substituting $\sigma=0$; we do not present them separately due to space constraints.

{\bf Remark} Theorem~\ref{thm:equivalent.bound} and Theorem~\ref{thm:mixed.partial} are applicable for any derivative order $\bm \beta \in\mathbb{N}_0^d$ as long as $|\bm\beta|\leq m$. For example, if $K\in C^2(\mX,\mX)$, these two theorems give learning rates for each component (\textit{e.g.}, the $j$th component) of the gradient $\nabla (\hat{f}_n-f_0)$ by setting the $j$th element of $\bm\beta$ to one and others to zero, in which case $\kappa_{\bm\beta}^2 = \sup_{\bx\in\mX}\partial_{x_j}\partial_{x_j}K(\bx,\bx)$ for $1\leq j\leq d$. Although our focus is on derivative estimation, setting $\bm\beta = \bm 0$ in the two theorems also provides error bounds for estimating the regression function using the KRR estimator; hence, we unify the study of the regression function and its derivatives in this article. There are existing error bounds for the KRR estimator $\hat{f}_n$. For example, Theorem 6 in \cite{smale2005shannon} established an RKHS bound for $\hat f_n - f_\lambda$ under the bounded sampling setting (\textit{i.e.}, the response $y_i$ is bounded), and their rate $O(1/\sqrt{n}\lambda)$ is similar to our \boundb{} with $\bm \beta = \bm 0$; \cite{yang2017frequentist} provided error bounds for $\hat{f}_n$ under the $L_\infty$ norm for kernels with polynomially decaying eigenvalues.

\boundb{} is applicable for any Mercer kernels, broadening the applicability of the proposed method. \bounda{} relies on the additional Assumption~\ref{ass:B}. When both \bounda{} and \boundb{} hold, the following result compares them by providing sufficient conditions under which \boundb{} cannot be tighter than \bounda{}.

\begin{corollary}\label{cor:sharp}
Take $\delta = n^{-10}$ in both \bounda{} and \boundb{}. If $\|f_\lambda - f_0\|_\infty = o(1)$, $\tilde{\kappa}_\lambda^2 = o(\sqrt{n/\log n})$ and $\tilde\kappa_{\bm\beta,\lambda}\tilde\kappa_\lambda = o(\lambda^{-1})$, then \bounda{} is asymptotically less than \boundb{}.
\end{corollary}

The three conditions in Corollary~\ref{cor:sharp} can be verified using special examples. For instance, considering the kernel and regression function in Theorem~\ref{thm:minimax.diff.matern} and invoking Lemma~\ref{lem:RKHS.derivative}, Lemma~\ref{lem:differentiability.matern} and Lemma \ref{lem:matern.deriv.deterministic}, $\|f_\lambda - f_0\|_\infty = o(1)$ and $\tilde\kappa_{\bm\beta,\lambda}\tilde\kappa_\lambda = o(\lambda^{-1})$ automatically hold, whereas $\tilde{\kappa}_\lambda^2 = o(\sqrt{n/\log n})$ is equivalent to $\lambda \gtrsim (\log n/n)^\alpha$. In particular, this condition on $\lambda$ is satisfied by the optimal value $({\log n}/{n})^{\frac{2\alpha}{2\alpha+1}}$ derived in Theorem~\ref{thm:minimax.diff.matern}.

Theorem~\ref{thm:equivalent.bound} and Theorem~\ref{thm:mixed.partial} lead to convergence rates of the plug-in KRR estimator $\partial^{\bm\beta}\hat{f}_n$ by invoking augmenting estimates of $\partial^{\bm\beta}f_{\lambda}-\partial^{\bm\beta}f_0$. We conclude this section with a few relatively abstract examples for such estimates, and introduce another more concrete example in detail in the next section.

\begin{theorem}\label{thm:dev.deterministic.bound}
Suppose Assumption~\ref{ass:A} holds and $\bm\beta\in\mathbb{N}_0^d$ with $|\bm\beta|\leq m$.
\begin{enumerate}[(a)]\itemsep0pt 
\item  Under Assumption~\ref{ass:B}, if $L_K^{-r}f_0\in \Ltwo$ for some $1/2<r\leq 1$, then it holds
\begin{equation}
\|\partial^{\bm\beta} f_\lambda-\partial^{\bm\beta} f_0\|_\infty\leq \kappa_{\bm\beta}\lambda^{r-1/2}\|L_K^{-r}f_0\|_{2}.
\end{equation}
\item Suppose that $K$ assumes eigendecomposition with respect to the Fourier basis and $L_K^{-r}f_0\in C^p(\mX)$ for some $0<r\leq 1$ and $p>d+|\bm\beta|$. Then there exists $C_3>0$ such that
\begin{equation}\label{eq:non.sample.bound.2}
\|\partial^{\bm\beta} f_\lambda-\partial^{\bm\beta} f_0\|_\infty\leq C_3\lambda^r\zeta(p-d+1-|\bm\beta|).
\end{equation}
\end{enumerate}
\end{theorem}

{\bf Remark} The condition $L_K^{-r}f_0\in \Ltwo$ is adopted from \cite{smale2007learning}, in which $r$ can be understood as a smoothness parameter of $f_0$. When $r=1/2$, the condition $L_K^{-1/2}f_0\in\Ltwo$ is equivalent to $f_0\in\bbH$. To see this, note that $\|L_K^{-1/2}f_0\|^2_{2}=\|\sum_{i=1}^{\infty}f_i\psi_i/\sqrt{\mu_i}\|^2_{2}=\sum_{i=1}^{\infty}f_i^2/\mu_i=\|f\|^2_\bbH$. Hence, part (a) of Theorem~\ref{thm:dev.deterministic.bound} provides a rate for $f_0\in\bbH$, while part (b) allows a wider range of $r$ and does not necessarily require $f_0\in\bbH$. But part (b) requires the image $L_K^{-r}f_0\in C^p(\mX)$. We next study a general setting where $L_K^{-r}f_0\in \Ltwo$ for $0<r\leq 1/2$.

We state the following assumption on an embedding property, which is also considered in \cite{fischer2020sobolev}.
\renewcommand{\theassumption}{C}
\begin{assumption}\label{ass:emb}
$\|L_K^{q/2} f\|_\infty \leq A \|f\|_2$ for $0<q\leq 1$, some constant $A>0$ and any $f\in\Ltwo$.
\end{assumption}
The larger $q$ is, the weaker the embedding property is. Assumption~\ref{ass:emb} always holds for $q = 1$ by noting that $\|L_K^{1/2} f\|_\infty \leq \kappa \|L_K^{1/2} f\|_{\bbH} = \kappa \|f\|_2$.

\begin{theorem}\label{thm:emb}
Under Assumptions \ref{ass:A} and \ref{ass:emb}, if $K$ assumes eigendecomposition with respect to the Fourier basis and $L_K^{-r}f_0\in \Ltwo$ for some $0<r\leq 1/2$, then there exists $C_4>0$ such that for any $\bm\beta\in\mathbb{N}_0^d$ with $|\bm\beta|\leq m$, 
\begin{equation}
\|\partial^{\bm\beta} f_\lambda-\partial^{\bm\beta} f_0\|_\infty\leq AC_4 \lambda^{r-q/2-q|\bm\beta|}\|g^*\|_2,
\end{equation}
where $g^*\in\Ltwo$ is determined by $f_0$, $r$ and $\bm\beta$.
\end{theorem}

\section{Nearly minimax optimal rate for H\"older and Sobolev class functions}
\label{sec:matern}

In this section, we demonstrate the sharpness of the established bounds using a concrete example. In particular, we use kernels with polynomially decaying eigenvalues for $K$ and consider $\mX=[0,1]$ along with a uniform sampling process for $p_X$ to ease presentation. We also assume that the eigenfunctions of the kernel $\{\psi_i\}_{i=1}^\infty$  are the Fourier basis:
\begin{equation}\label{eq:fourier}
\psi_1(x)=1,\ \psi_{2i}(x)=\cos(2\pi ix),\ \psi_{2i+1}=\sin(2\pi ix), i\in\mathbb{N}, 
\end{equation}
which clearly satisfies Assumption~\ref{ass:B}. We formalize the definition of such kernels as follows.
\begin{definition}\label{def:kalpha}
A kernel with polynomially decaying eigenvalues $K_\alpha:[0,1]\times[0,1]\rightarrow\bbR$ assumes an eigendecomposition with respect to the Lebesgue measure $\mu$ such that the eigenvalues $\mu_i\asymp i^{-2\alpha}$ for some $\alpha>0$ and the eigenfunctions $\{\psi_i\}_{i=1}^\infty$ are the Fourier basis functions.
\end{definition}
Examples of kernels with polynomially decaying eigenvalues include the Mat\'ern kernel and Sobolev kernel \citep{wahba1990spline,gu2013smoothing}. For example, it is well known that the eigenvalues of Mat\'ern kernel with parameter $\nu$ satisfy $\mu_i\asymp i^{-2(\nu+1/2)}$ for $i\in\mathbb{N}$. In the following, we consider two function classes, a H\"older class $H^{\alpha}[0, 1]$ and a Sobolev class $S^{\alpha}[0,1]$,  for the true regression function $f_0$.
\begin{definition}
Let $\{\psi_i\}_{i=1}^{\infty}$ be the Fourier basis of $L^2_\mu[0,1]$ in \eqref{eq:fourier}. For any $\alpha>0$, the H\"older class $H^{\alpha}[0,1]$ is a Hilbert space defined as
\begin{equation}
H^{\alpha}[0,1]=\bigg\{f\in L^2_\mu[0,1]: \|f\|^2_{H^{\alpha}[0,1]}=\sum_{i=1}^{\infty}i^{\alpha}|f_i|<\infty\bigg\},
\end{equation}
where $f_i=\left<f,\psi_i\right>_2$.
\end{definition}
For any $f\in H^{\alpha}[0,1]$, $f$ lies in the \textit{$\alpha$-smooth H\"older space}, \textit{i.e.}, it has continuous derivatives up to order $\lfloor\alpha\rfloor$ and the $\lfloor\alpha\rfloor$th derivative is Lipschitz continuous of order $\alpha-\lfloor\alpha\rfloor$ \citep{yang2017frequentist}. To see this, note that
\begin{align}
|f^{\lfloor\alpha\rfloor}(x) - f^{\lfloor\alpha\rfloor}(x')| &= \left|\s f_i (\psi_i^{\lfloor\alpha\rfloor}(x) - \psi_i^{\lfloor\alpha\rfloor}(x'))\right| \lesssim \s |f_i| i^{\lfloor\alpha\rfloor}\\
& \lesssim \s |f_i| i^{\lfloor\alpha\rfloor} (i|x-x'|)^{\alpha - \lfloor\alpha\rfloor} \lesssim |x-x'|^{\alpha - \lfloor\alpha\rfloor}.
\end{align}

\begin{definition}
Let $\{\psi_i\}_{i=1}^{\infty}$ be the Fourier basis of $L^2_\mu[0,1]$ in \eqref{eq:fourier}. For any $\alpha>1/2$, the Sobolev class $S^{\alpha}[0,1]$ is a Hilbert space defined as
\begin{equation}
S^{\alpha}[0,1]=\bigg\{f\in L^2_\mu[0,1]: \|f\|^2_{S^{\alpha}[0,1]}=\sum_{i=1}^{\infty}i^{2\alpha}f_i^2<\infty\bigg\},
\end{equation}
where $f_i=\left<f,\psi_i\right>_2$.
\end{definition}
For $\alpha\in\mathbb{N}$, functions in $S^{\alpha}[0,1]$ belong to the \textit{$\alpha$-smooth Sobolev space} (cf. Theorem 7.11 in \cite{wasserman2006all}), which consists of functions with absolutely continuous $\alpha-1$ derivatives and whose $\alpha$th derivative has uniformly bounded $L_2$ norm and is also a Hilbert space.

It is easy to see that $H^\alpha[0,1] \subset S^\alpha[0,1]$ by definition. On the other hand, $S^{\alpha+\alpha_0}[0,1] \subset H^\alpha[0,1]$ for $\alpha_0>1/2$. Indeed, it holds that 
\begin{equation}
i^{\alpha}|f_i| \leq i^{-2\alpha_0} + i^{2(\alpha+\alpha_0)}f_i^2
\end{equation}
for all $i\in\mathbb{N}$.
Hence, if $f \in S^{\alpha + \alpha_0}[0, 1],$ we have 
\begin{equation}
\s i^{\alpha}|f_i| \leq \s i^{-2\alpha_0} + \s i^{2(\alpha+\alpha_0)}f_i^2 < \infty.
\end{equation}

Considering the equivalent kernel $\tilde{K}_\alpha$ of $K_\alpha$. Let the higher-order analog of the effective dimension for $K_\alpha$ be $\tilde{\kappa}_{\alpha,m,\lambda}^2$ for $m\in\mathbb{N}_0$, where the subscript $\alpha$ emphasizes the use of $K_{\alpha}$ compared to the general definition in \eqref{eq:high.order.kappa}. Similarly to the preceding section, all results in this section cover the regression function as a special case with $m = 0$. For example, we allow $m=0$ in $\tilde{\kappa}_{\alpha,m,\lambda}$, which corresponds to $\tilde{\kappa}_{\alpha,0,\lambda}^2=\tilde{\kappa}_{\alpha,\lambda}^2={\sup_{x\in[0,1]}\tilde{K}_\alpha(x,x)}$.  

Lemma~\ref{lem:differentiability.matern} provides the differentiability of $K_\alpha$ and the exact order of $\tilde{\kappa}_{\alpha,m,\lambda}$ with respect to $\lambda$.
\begin{lemma}\label{lem:differentiability.matern}
If $\alpha>m+1/2$ for $m\in\mathbb{N}_0$, then $K_\alpha \in C^{2m}([0,1]\times[0,1])$, $\tilde{K}_\alpha\in C^{2m}([0,1]\times[0,1])$ and $\tilde{\kappa}_{\alpha,m,\lambda}^2\asymp \lambda^{-\frac{2m+1}{2\alpha}}$.
\end{lemma}

Thus, $K_\alpha$ satisfies Assumption~\ref{ass:A} whenever $\alpha > m + 1/2$. The $1/2$ gap between $m$ and $\alpha$ appears to be smaller than those required by existing literature; for example, local polynomial and smoothing splines often require the regression function $f_0 \in C^{m+1}(\mX)$ when estimating the $m$th derivative \citep{de2013derivative, charnigo2011generalized}, and difference-based methods such as \cite{wang2015derivative} assumed the true regression function to be five times differentiable when estimating the first derivative.

The next lemma studies the differentiability of functions in the RKHS $\bbH_\alpha$ induced by $K_\alpha$. It turns out that the equivalent RKHS norm upper bounds the $L_2$ norm of the derivatives. Note that $\bbH_\alpha$ and $\tilde{\bbH}_\alpha$ consist of the same class of functions, thus sharing the same differentiability property. 
\begin{lemma}\label{lem:matern.RKHS.derivative}
If $\alpha>m+1/2$ for $m\in\mathbb{N}_0$, then $\bbH_\alpha \subset  C^m[0,1]$. Moreover, there exists a constant $C_m>0$ that does not depend on $\lambda$ such that $\|f^{(m)}\|_2\leq C_m\tilde{\kappa}_{\alpha,\lambda}^{-1}\tilde{\kappa}_{\alpha,m,\lambda}\|f\|_{\tilde{\bbH}_\alpha}$ \ for any $f\in\bbH_\alpha$.
\end{lemma}

As we have seen in Theorem~\ref{thm:equivalent.bound} and Theorem~\ref{thm:mixed.partial}, calculating the learning rate of $\hat f_n^{(m)}-f_0^{(m)}$ requires the rate of $f_\lambda^{(m)}-f_0^{(m)}$. We next provide the error bound for this quantity under the $\tilde{\bbH}_\alpha$ norm.

\begin{lemma}\label{lem:matern.deriv.deterministic}
Suppose $f_0\in H^{\alpha}[0,1]$ or $f_0\in S^{\alpha}[0,1]$ for $\alpha>1/2$. If the kernel is chosen to be $K_\alpha$, then it holds that
\begin{equation}
\|f_\lambda-f_0\|_{\tilde{\bbH}_\alpha} \lesssim\lambda^{\frac{1}{2}}.
\end{equation}
\end{lemma}

When $m=0$, the three lemmas above provide error bounds for estimating the regression function. In particular, Lemma~\ref{lem:differentiability.matern} implies $\tilde{\kappa}_{\alpha,\lambda}^2\asymp \lambda^{-\frac{1}{2\alpha}}$, Lemma~\ref{lem:matern.RKHS.derivative} gives $\|f\|_2\leq \|f\|_{\tilde{\bbH}_\alpha}$, and Lemma~\ref{lem:matern.deriv.deterministic} leads to $\|f_\lambda-f_0\|_{\tilde{\bbH}_\alpha} \lesssim\lambda^{\frac{1}{2}}$.

We are now in a position to present a non-asymptotic convergence rate of $\hat{f}_n^{(m)}$.

\begin{theorem}\label{thm:minimax.diff.matern}
Suppose $f_0\in H^{\alpha}[0,1]$ or $f_0\in S^{\alpha}[0,1]$ for $\alpha>m+1/2$ and $m\in\mathbb{N}_0$, and the kernel is chosen to be $K_\alpha$. Then it holds with $\PP_0^{(n)}$-probability at least $1-n^{-10}$ that
\begin{equation}
\|\hat{f}_n^{(m)}-f_0^{(m)}\|_2\lesssim \left(\frac{\log n}{n}\right)^{\frac{\alpha-m}{2\alpha+1}},
\end{equation}
with the corresponding choice of regularization parameter $\lambda\asymp ({\log n}/{n})^{\frac{2\alpha}{2\alpha+1}}$.
\end{theorem}

Theorem~\ref{thm:minimax.diff.matern} yields that the plug-in KRR estimator is minimax optimal up to a logarithmic factor for estimating $f_0^{(m)}$ with $f_0 \in H^\alpha[0,1]$ or $f_0\in S^{\alpha}[0,1]$ under the $L_2$ norm. To see this, first consider $f_0 \in H^\alpha[0,1]$ and let $\epsilon_{\alpha,m,n}$ be the minimax optimal rate for estimating $f_0^{(m)}$. Note that $n^{-\frac{\alpha-m}{2\alpha+1}}$ is the optimal rate for estimating the $m$th derivative of $\alpha$-smooth H\"older functions \citep{stone1982optimal}. Let $H_{per}^\alpha[0,1]$ denote the subset of $\alpha$-smooth H\"older space that satisfies the periodic boundary condition $f^{(j)}(0) = f^{(j)}(1)$ for $0\leq j \leq \lfloor \alpha \rfloor -1$. It can be shown that the corresponding $L_2$ minimax rate for $H_{per}^\alpha[0,1]$ is also $n^{-\frac{\alpha-m}{2\alpha+1}}$. Suppose $f \in H_{per}^{\alpha'}[0,1]$ for $\alpha' > \alpha$, according to Theorem 12.20 in \cite{gockenbach2005partial} and Theorem 3.14 in \cite{cuddy2012convergence}, the Fourier coefficient of $f$ satisfies $f_i = O(i^{-\alpha'-1})$. Consequently, $H_{per}^{\alpha'}[0,1] \subset H^\alpha[0,1]$, and hence $\epsilon_{\alpha,m,n} \geq n^{-\frac{\alpha'-m}{2\alpha'+1}}$. Letting $\alpha' \downarrow \alpha$, we have $\epsilon_{\alpha,m,n} \geq n^{-\frac{\alpha-m}{2\alpha+1}}$, indicating the near minimax optimality of the plug-in KRR estimator when the function class is $H^\alpha[0,1]$. The same minimax optimality extends to the function class $S^\alpha[0,1]$. When $m=0$, it is known that the optimal rate for estimating $f_0\in S^\alpha[0,1]$ is $n^{-\frac{\alpha}{2\alpha+1}}$ (cf. Theorem 7.32 in \cite{wasserman2006all}). For a general derivative order $m \geq 0$, noting that $H^\alpha[0,1] \subset S^\alpha[0,1]$, the minimax rate for $S^\alpha[0,1]$ is no faster than that for $H^\alpha[0,1]$.
Since the plug-in KRR estimator achieves the same convergence rate for both classes as shown in Theorem~\ref{thm:minimax.diff.matern}, we arrive at the conclusion that the proposed KRR estimator is also nearly minimax optimal when estimating $f_0^{(m)}$ with $f_0 \in S^{\alpha}[0,1]$.

{\bf Remark} We can see that given the smoothness level of the regression function, the rate-optimal estimation of the derivatives shares the same choice of $\lambda$ between various derivative orders. Thus, the plug-in KRR estimator is \textit{adaptive} to the order of the derivative to be estimated. This indicates that the proposed estimator enjoys the so-called \textit{plug-in property}~\citep{bickel2003nonparametric}, a phenomenon in which a rate-optimal nonparametric estimator also efficiently estimates some bounded linear functionals. Instead of bounded linear functionals, we establish the plug-in property for function derivatives.

{\bf Remark} The adaptivity and plug-in property of the plug-in KRR estimator are in sharp contrast to some existing methods. The minimax optimal rate for a specific derivative order can be achieved by various methods in the literature but with caveats, including difference-based methods \citep{wang2015derivative,dai2016optimal,liu2020smoothed}, local polynomial regression~\citep{fan1996local,delecroix1996nonparametric}, and smoothing splines~\citep{stone1985additive, zhou2000derivative}. For difference-based methods, the optimality only applies to interior points, and boundary correction is required for both practical implementation and theoretical understanding.  Difference-based methods are typically used to estimate the first two derivatives and require more tuning parameters for a higher derivative order. For local polynomial regression and smoothing splines, as pointed out by \cite{wahba1990optimal} and \cite{charnigo2011generalized}, the optimal choice of smoothing parameter depends on the order of the derivative. Hence, they do not enjoy the aforementioned plug-in property for function derivatives while the proposed plug-in KRR method does. In other words, for local polynomial regression and smoothing splines, when the estimator achieves the optimal rate of convergence for the regression function $f_0$, the plug-in derivative estimators with the same tuning parameter values will be sub-optimal, and \textit{vice versa}. The lack of adaptivity to derivative orders in these existing methods renders parameter tuning challenging in the presence of varying derivative orders. 

{\bf Remark} There has been a rich literature on theoretical guarantees of KRR, but much of the focus has been on regression functions \citep{cucker2002best,zhang2005learning,caponnetto2007optimal,steinwart2009optimal,mendelson2010regularization,yang2017frequentist}. For example, \cite{yang2017frequentist} derived error bounds for $\hat{f}_n$ when the regression function belongs to similar function classes. Our focus in this section is instead on derivative estimation, which requires different assumptions and techniques. The framework developed in this article for estimating function derivatives, including the adaptivity and plug-in property of the plug-in KRR estimators, might provide useful insights and results that could be helpful in building upon existing bounds on KRR estimators for derivatives.

\section{Practical consideration}\label{sec:practical}

Estimating the derivatives of the regression function has a wide range of applications in many areas. In the regression setting considered in this article, function derivatives have direct real-world applications. For example, estimating function derivatives is directly useful in understanding the behavior of the hypothesized dark energy equation in cosmology \cite{holsclaw2013gaussian}, which is a function of the second derivative of the data process. In ocean sciences, the derivative function provides the rate of sea-level change at a particular time point \citep{cahill2015modeling}, offering insights into the evolution of dynamic sea-level rise over time. In more general settings, derivatives are frequently used in spatial process models \citep{banerjee2003directional} and shape-constrained regression that utilizes derivative processes~\citep{riihimaki2010gaussian,wang2016estimating}, and may improve the computational efficiency for nonlinear dynamic system identification \citep{solak2003derivative}, while serving as a tool in detecting local extrema \citep{song2006nonparametric,li2021semiparametric} and efficient modeling of functional data \citep{dai2018derivative}. An important implication of our developed theory and methods is that KRR estimators as a common method for unknown regression functions can be used to infer derivatives of functions by the simple plug-in strategy, with easy tuning and explicit expressions.

In practice, one needs to choose the kernel and regularization parameter $\lambda$, and account for computational complexity. 

{\textbf{Kernel selection.} There is a rich menu for the covariance kernel \cite[Chapter 4]{rasmussen2006}, and below we introduce several choices with polynomially decaying eigenvalues. Such kernels are commonly used in the literature \citep{amini2012sampled,zhang2015divide,yang2017frequentist}. The Mat\'ern kernel is given by
\begin{equation}\label{eq:matern.kernel}
K_{\text{Mat},\nu}(x,x')=\frac{2^{1-\nu}}{\Gamma(\nu)}\left(\sqrt{2\nu}|x-x'|\right)^\nu B_\nu\left(\sqrt{2\nu}|x-x'|\right),
\end{equation}
where $B_\nu(\cdot)$ is the modified Bessel function of the second kind with smoothness parameter $\nu$ to be determined. It is well known that the eigenvalues of the Mat\'ern kernel satisfy that $\mu_i\asymp i^{-2(\nu+1/2)}$ for $i\in\mathbb{N}$. In practice, we select $\nu$ via leave-one-out cross validation and minimize the mean square error of the regression function. The Sobolev kernel is another class of kernels with polynomial decaying eigenvalues that underlie the Sobolev
spaces with different orders of smoothness \citep{birman1967piecewise,gu2013smoothing}. In our numerical experiment we consider the second-order Sobolev kernel
\begin{equation}
K_{\text{Sob}}(x,x')=1+xx'+\min\{x,x'\}^2(3\max\{x, x'\}-\min\{x, x'\})/6,
\end{equation}
which generates an RKHS of Lipschitz functions with smoothness $\alpha=2$. Other higher-order Sobolev kernels
also exhibit polynomial eigendecay with larger smoothness levels. Choosing the covariance kernel can be largely assisted by domain knowledge in many fields as each kernel encodes various properties of its samples from the induced RKHS. For example, the squared exponential covariance kernel as the limiting case of the Mat\'ern kernel with $\nu \rightarrow \infty$ is a popular choice in event-related potential analysis in neuroscience \citep{yu2023bayesian} thanks to the induced smooth functions that agree with domain knowledge, and similarly, Mat\'ern kernels are more popular for less smooth functions such as in spatial process models \citep{banerjee2003directional}. Shape constraints also point to specific kernels; for example, inference on periodic functions necessitates choosing kernels defined on spheres that encode periodicity \citep{li2017bayesian}. One can also resort to cross validation as a data-driven solution to choose a kernel among multiple options using a model selection perspective.

\textbf{Parameter tuning.} For a given kernel and under normal assumptions, we estimate error variance $\sigma^2$
by its maximum marginal likelihood estimator (MMLE)
\begin{equation}
\hat\sigma_n^2 := \lambda \by^T [K(X,X) + n\lambda \bI_n]^{-1}\by
\end{equation}
and choose the regularization parameter $\lambda$ by maximizing the marginal likelihood
\begin{equation}
\by\mid X\sim N(0, \hat\sigma_n^2(n\lambda)^{-1}K(X,X)+\hat\sigma_n^2\bI_n).
\end{equation}
These parameters are used for \textit{any} order of derivatives in view of the order adaptive property of the proposed method; in contrast, optimal parameter tuning in competing methods may vary with the derivative order and deviate from the one chosen for estimating the regression function. We will assess the finite-sample performance of plug-in KRR estimators with this choice of $\lambda$ in the next section. 

The above method for parameter tuning is also known as the empirical Bayes approach. We advocate for this approach because of its empirical success in the Bayesian literature, and an established equivalence link between Bayesian and non-Bayesian frameworks that allows us to transfer concepts from the Bayesian regime to kernel ridge regression \citep{liu2020equivalence}. In other settings, we have noticed that the MMLE of $\lambda$ tends to adapt to the unknown smoothness level of the underlying function when paired with an oversmooth kernel. That is, the MMLE of $\lambda$ often leads to excellent performance when the kernel's smoothness level is equal to or greater than that of the regression function. This suggests that an alternative effective strategy for estimating smooth functions with unknown smoothness and their derivatives could involve deploying an oversmooth kernel, such as the squared exponential kernel, and choosing $\lambda$ via the MMLE. The use of oversmooth kernels, including the squared exponential kernel, is consistent with existing literature such as \cite{bach2013sharp}. A formal investigation of this practically appealing adaptivity feature of the MMLE is an interesting future work, and part of our efforts in this direction will be reported in \cite{liu2022optimal}.

\textbf{Computational complexity.} The proposed method has analytical forms for any order of derivatives, facilitating fast implementation. The average total running time of the proposed method is 0.31 when $n=100$ and 0.97 seconds when $n=500$ in R on a PC with 2.3 GHz 8-Core Intel Core i9 CPU. Computing the eigendecomposition of $K(X, X)$ typically takes $O(n^3)$ times, but this is a one-time cost as we can store the eigendecomposition of $K(X,X)$ to speed up the calculation of $[K(X, X) + n \lambda \bm{I}_n]^{-1}$ for any given $\lambda$. The subsequent estimation process consists of two steps. First, we use limited-memory bound constrained ``BFGS'' in the ``optim'' function in R to find the optimal $\lambda$. This tuning step has complexity $O(kn^2)$, where $k$ is the number of iterations, and is finished within an average of 0.28 seconds when $n=100$ and 0.56 seconds when $n=500$. The following step is to calculate the plug-in KRR estimate given the optimal $\lambda$, which has complexity $O(n^2)$.

\section{Simulation}\label{sec:simulation}

In this section, we assess the finite sample performance of the plug-in KRR estimator relative to several methods and provide numerical evidence of its agreement with the minimax optimal rate. 

\subsection{Comparison with existing methods} 
We consider two regression functions:
$f_{01}(x)=\exp\{-4(1-2x)^2\}(1-2x)$ and $f_{02}(x)=\sin(8x)+\cos(8x)+\log(4/3+x)$ for $ x\in [0,1]$,
with random design $X_i\sim \text{Unif}[0,1]$ and sample size $n = 500$. We generate the response $\by$ following Model~\eqref{eq:model} by adding Gaussian error $\varepsilon_i\sim N(0, 0.2^2)$ to $f_{01}$ and $f_{02}$. We consider up to the third derivative to accommodate competing methods, but note that the proposed method is readily available for any order.  We also conduct simulations under fixed design; the comparison is similar, and the results are deferred to the Supplementary Material. 

For the proposed method, we use the second-order Sobolev kernel and Mat\'ern kernel given in Section~\ref{sec:practical}. We compare the plug-in KRR estimator with three other methods: local polynomial regression with degree $p = 2$ (R package ‘locpol’ in \cite{cabrera2012locpol}), penalized smoothing spline (R package ‘pspline’ in \cite{ramsay2020pspline}) and locally weighted least  squares regression (coded as ‘LowLSR') proposed by \cite{wang2019robust}. For local polynomial regression, we use the Gaussian kernel and select the bandwidth via cross validation. For smoothing spline, we use cubic penalized smoothing spline with other parameters set to the default values. When implementing LowLSR, we set the number of difference quotients $k$ to 50 for the first derivative and increase it to 100 for the second derivative, leading to 400 and 300 non-boundary points, respectively. We remark that it is not easy for LowLSR to estimate high-order derivatives, and we only use it to estimate the first two derivatives.

We conduct a Monte Carlo study with 100 repetitions. We evaluate each estimator except LowLSR at 100 equally spaced points in $[0,1]$, and calculate the root mean square error (RMSE):
\begin{equation}
\mathrm{RMSE} = \sqrt{\frac{1}{100} \sum_{t = 0}^{99} \{\hat{s}(t/99) - s(t /99)\}^2}, 
\end{equation}
where $\hat{s}$ is the estimated function and $s$ the true function ($f^{(m)}_{01}$ or $f^{(m)}_{02}$ for $k=1,2,3$). Since LowLSR does not allow evaluation at boundary points or points different from the observed $X_i$, we compute the RMSE at every other 5 points from the sorted $X_i$ that are away from the boundaries, resulting in 80 and 60 testing points for first and second derivative estimation, respectively.

\begin{figure}
\centering
\begin{tabular}{cc}
\begin{tikzpicture}
\begin{axis}[
width      = 0.5\textwidth,
height     = 0.35\textwidth,
boxplot/draw direction=y,
boxplot/every box/.style={fill=gray!50},
boxplot/box extend=0.8,
xticklabel style = {align=center, font=\small, rotate=60},
xtick={1,2,3,4,5},
xticklabels={
Sobolev, Mat\'ern , locpol, pspline, LowLSR},
ymin = -0.1,
ymax = 1.6,
ytick={0,0.5,1,1.5},
yticklabels={0,0.5,1,1.5},
]
\addplot[boxplot] table[y index=0] {box-11.dat};
\addplot[boxplot] table[y index=1] {box-11.dat};
\addplot[boxplot] table[y index=2] {box-11.dat};
\addplot[boxplot] table[y index=3] {box-11.dat};
\addplot[boxplot] table[y index=4] {box-11.dat};
\end{axis}
\end{tikzpicture}
\begin{tikzpicture}
\begin{axis}[
width      = 0.5\textwidth,
height     = 0.35\textwidth,
boxplot/draw direction=y,
boxplot/every box/.style={fill=gray!50},
boxplot/box extend=0.8,
xticklabel style = {align=center, font=\small, rotate=60},
xtick={1,2,3,4,5},
xticklabels={
Sobolev, Mat\'ern , locpol, pspline, LowLSR},
ymin = -0.1,
ymax = 2.6,
ytick={0,0.5,1,1.5,2,2.5},
yticklabels={0,0.5,1,1.5,2,2.5},
]
\addplot[boxplot] table[y index=0] {box-21.dat};
\addplot[boxplot] table[y index=1] {box-21.dat};
\addplot[boxplot] table[y index=2] {box-21.dat};
\addplot[boxplot] table[y index=3] {box-21.dat};
\addplot[boxplot] table[y index=4] {box-21.dat};
\end{axis}
\end{tikzpicture}
\end{tabular} \vspace{-.75\baselineskip}
\caption{Boxplots of RMSEs: $f'_{01}$ (left) and $f'_{02}$ (right).}
\label{fig:boxplot}
\vspace{\baselineskip}
\begin{tabular}{cc}
\begin{tikzpicture}
\begin{axis}[
width      = 0.5\textwidth,
height     = 0.35\textwidth,
boxplot/draw direction=y,
boxplot/every box/.style={fill=gray!50},
boxplot/box extend=0.8,
xticklabel style = {align=center, font=\small, rotate=60},
xtick={1,2,3,4,5},
xticklabels={
Sobolev, Mat\'ern , locpol, pspline, LowLSR},
ymin = -1,
ymax = 21,
ytick={0,5,10,15,20},
yticklabels={0,5,10,15,20},
]
\addplot[boxplot] table[y index=0] {box-12.dat};
\addplot[boxplot] table[y index=1] {box-12.dat};
\addplot[boxplot] table[y index=2] {box-12.dat};
\addplot[boxplot] table[y index=3] {box-12.dat};
\addplot[boxplot] table[y index=4] {box-12.dat};
\end{axis}
\end{tikzpicture}
\begin{tikzpicture}
\begin{axis}[
width      = 0.5\textwidth,
height     = 0.35\textwidth,
boxplot/draw direction=y,
boxplot/every box/.style={fill=gray!50},
boxplot/box extend=0.8,
xticklabel style = {align=center, font=\small, rotate=60},
xtick={1,2,3,4,5},
xticklabels={
Sobolev, Mat\'ern , locpol, pspline, LowLSR},
ymin = -2,
ymax = 52,
ytick={0,10,20,30,40,50},
yticklabels={0,10,20,30,40,50},
]
\addplot[boxplot] table[y index=0] {box-22.dat};
\addplot[boxplot] table[y index=1] {box-22.dat};
\addplot[boxplot] table[y index=2] {box-22.dat};
\addplot[boxplot] table[y index=3] {box-22.dat};
\addplot[boxplot] table[y index=4] {box-22.dat};
\end{axis}
\end{tikzpicture}
\end{tabular} \vspace{-.75\baselineskip}
\caption{Boxplots of RMSEs: $f''_{01}$ (left) and $f''_{02}$ (right).}
\label{fig:boxplot-2}
\vspace{\baselineskip}
\begin{tabular}{cc}
\begin{tikzpicture}
\begin{axis}[
width      = 0.5\textwidth,
height     = 0.35\textwidth,
boxplot/draw direction=y,
boxplot/every box/.style={fill=gray!50},
boxplot/box extend=0.8,
xticklabel style = {align=center, font=\small, rotate=60},
xtick={1,2,3,4},
xticklabels={
Sobolev, Mat\'ern , locpol, pspline},
ymin = -10,
ymax = 310,
ytick={0,50,100,150,200,250,300},
yticklabels={0,50,100,150,200,250,300},
]
\addplot[boxplot] table[y index=0] {box-13.dat};
\addplot[boxplot] table[y index=1] {box-13.dat};
\addplot[boxplot] table[y index=2] {box-13.dat};
\addplot[boxplot] table[y index=3] {box-13.dat};
\end{axis}
\end{tikzpicture}
\begin{tikzpicture}
\begin{axis}[
width      = 0.5\textwidth,
height     = 0.35\textwidth,
boxplot/draw direction=y,
boxplot/every box/.style={fill=gray!50},
boxplot/box extend=0.8,
xticklabel style = {align=center, font=\small, rotate=60},
xtick={1,2,3,4},
xticklabels={
Sobolev, Mat\'ern , locpol, pspline},
ymin = -100,
ymax = 2100,
ytick={0,500,1000,1500,2000},
yticklabels={0,500,1000,1500,2000},
]
\addplot[boxplot] table[y index=0] {box-23.dat};
\addplot[boxplot] table[y index=1] {box-23.dat};
\addplot[boxplot] table[y index=2] {box-23.dat};
\addplot[boxplot] table[y index=3] {box-23.dat};
\end{axis}
\end{tikzpicture}
\end{tabular} \vspace{-.75\baselineskip}
\caption{Boxplots of RMSEs: $f'''_{01}$ (left) and $f'''_{02}$ (right). LowLSR is not applicable to estimate the third derivative.}
\label{fig:boxplot-3}
\end{figure}

Figure~\ref{fig:boxplot} presents the boxplot of RMSEs for estimating $f'_{01}$ and $f'_{02}$ for each method. We can see that for $f'_{02}$ KRR with Mat\'ern kernel achieves the lowest median RMSE among all methods, while KRR with Sobolev kernel is comparable to penalized smoothing spline and outperforms the other two methods. For $f'_{01}$, we observe a similar result with the relative position switched between the Sobolev kernel and Mat\'ern kernel. For both functions, LowLSR exhibits the highest median and the most variability of RMSE; this might be partly because LowLSR is designed for the fixed design setting. We consider a fixed design simulation in the Supplementary Material, in which LowLSR improves but still gives considerably larger RMSE than the better method of the two KRR estimators. Overall, the plug-in KRR estimator with Mat\'en kernel leads to the best RMSE for $f'_{02}$, and gives close results to the leading approach KRR with Sobolev kernel for $f'_{01}$.

Figure~\ref{fig:boxplot-2} presents the boxplot of RMSEs for each method when estimating $f''_{01}$ and $f''_{02}$. KRR with Mat\'ern kernel achieves the lowest median RMSE among all methods.

Figure~\ref{fig:boxplot-3} displays the boxplot of RMSEs for estimating $f'''_{01}$ and $f'''_{02}$. We can see that the performance of penalized smoothing spline is significantly worsened with high variability and the largest median RMSE, indicating the challenge when estimating high-order derivatives. The two KRR methods continue to give the best RMSEs, confirming our theory that the proposed estimator is adaptive to the derivative order. Comparing these two plug-in KRR estimators suggests the Mat\'en kernel leads to either similar or better RMSE, and appears to be the recommended choice under our simulation settings.

Figure~\ref{fig:function} displays the result from one random run in the Monte Carlo study for estimating the first derivatives. It can be seen that locpol and LowLSR do not perform well for estimating $f_{01}'$, while all methods produce relatively satisfactory results for estimating $f_{02}'$. KRR with either kernel estimates $f'_{01}$ fairly well, while the Sobolev kernel slightly underperforms Mat\'ern kernel when estimating the boundaries of $f'_{02}$.

\begin{figure}[H]
\centering
\begin{tabular}{cc}
\begin{tikzpicture}
\begin{axis}[
width      = 0.48\textwidth,
height     = 0.4\textwidth,
xlabel = $x$,
ylabel = $f'_{01}(x) \text{, } \hat f'_{01}(x)$,
xmin = -0.05,
xmax = 1.05,
ymin = -3.2,
ymax = 2.2,
ytick      = {-3, -2, -1, 0, 1, 2},
yticklabels= {-3, -2, -1, 0, 1, 2},
]
\addplot[mark=none,  black,   ultra thick] table[x index = 0, y index = 1]{plot-1.dat};
\addplot[mark=none,  green, very thick, dashed] table[x index = 0, y index = 2]{plot-1.dat};
\addplot[mark=none,  red, very thick, dashed] table[x index = 0, y index = 3]{plot-1.dat};
\addplot[mark=none,  blue, very thick, dash pattern={on 10pt off 4pt}] table[x index = 0, y index = 4]{plot-1.dat};
\addplot[mark=none,  yellow, very thick, dash pattern={on 10pt off 4pt}] table[x index = 0, y index = 5]{plot-1.dat};
\addplot[mark=none,  gray, very thick, dash pattern={on 10pt off 4pt}] table[x index = 0, y index = 1]{plot-1-supp.dat};
\end{axis}
\end{tikzpicture}

\begin{tikzpicture}
\begin{axis}[
width      = 0.48\textwidth,
height     = 0.4\textwidth,
xlabel = $x$,
ylabel = $f'_{02}(x) \text{, } \hat f'_{02}(x)$,
xmin = -0.05,
xmax = 1.05,
ymin = -16,
ymax = 16,
ytick      = {-15, -10, -5, 0, 5, 10, 15},
yticklabels= {-15, -10, -5, 0, 5, 10, 15},
]
\addplot[mark=none,  black,   ultra thick] table[x index = 0, y index = 1]{plot-2.dat};
\addplot[mark=none,  green, very thick, dashed] table[x index = 0, y index = 2]{plot-2.dat};
\addplot[mark=none,  red, very thick, dashed] table[x index = 0, y index = 3]{plot-2.dat};
\addplot[mark=none,  blue, very thick, dash pattern={on 10pt off 4pt}] table[x index = 0, y index = 4]{plot-2.dat};
\addplot[mark=none,  yellow, very thick, dash pattern={on 10pt off 4pt}] table[x index = 0, y index = 5]{plot-2.dat};
\addplot[mark=none,  gray, very thick, dash pattern={on 10pt off 4pt}] table[x index = 0, y index = 1]{plot-2-supp.dat};
\end{axis}
\end{tikzpicture}
\end{tabular} \vspace{-.75\baselineskip}
\caption{Estimates of $f'_{01}$ (left) and $f'_{02}$ (right) in one simulation: true derivative (full line); KRR with Sobolev kernel (green dash), Mat\'ern kernel (red dash); locpol (blue long dash), spline (yellow long dash) and LowLSR (grey long dash).}
\label{fig:function}
\end{figure}

\newcommand{\error}{\text{error}}
\subsection{Finite-sample comparison with minimax bounds} 
We next perform experiments as proof of concepts that the derived upper bound is observed in practice. To this end, we examine how the empirical error scales with the sample size, with $\lambda$ selected by maximizing its marginal likelihood. We consider the true regression function $f_0(x)=\sqrt{2}\s i^{-5}\sin i\cos[(i-1/2)\pi x]$ for $x\in[0,1]$, which belongs to $H^\alpha[0,1]$ with $\alpha=4$. Hence, we use a Mat\'ern kernel with $\nu=3.5$. We simulate $n_i$ observations from the regression model \eqref{eq:model} with $\varepsilon_i\sim N(0, 0.1)$ and $X_i\sim \text{Unif}[0,1]$. The sample size $n_i$ varies from 10 to 500 such that $\log(n_i)$'s are 100 equally spaced points in $[\log(10), \log(500)]$. We replicate the simulation 100 times for each sample size $n_i$. We then compute the average RMSE $\error_i$ of the 100 replications as an estimate of the $L_2$ error $\|\hat f_{n_i}' - f_0'\|_2$. The minimax optimal rate for estimating $f_0'$ is $n^{-1/3}$ under the $L_2$ norm \citep{stone1982optimal}. If our plug-in estimator is able to achieve this optimal rate, then the scatterplot of $(\log(n_i), \log(\error_i))$ should come close to forming a straight line $\log(\error_i) = -\frac{1}{3}\log(n_i) + \text{constant}$.

The left panel of Figure~\ref{fig:loglog} plots $\log(\error_i)$ versus $\log(n_i)$. The reference line in red has slope $-1/3$; its intercept is determined by least square fitting with fixed slope $-1/3$, which is $\sum_{i = 1}^{100} \{\log(\error_i) + \frac{1}{3} \log (n_i)\}/100$. We can see that the points are distributed around the line, suggesting that the estimation error of our plug-in KRR estimator agrees with the theoretical minimax rate. To investigate the effect of sample size more dynamically, the right panel of Figure~\ref{fig:loglog} shows the rolling least square slopes with moving windows of 40 observations, \textit{i.e.}, the $k$th slope in the plot is obtained by linear regression using data $\{(\log(n_i), \log(\error_i)) : k\leq i\leq k+39\}$ for $1\leq k \leq 61$. The slopes are close to the reference line (in red) that represents the minimax rate $-1/3$ for all the sample sizes under consideration, and we do not observe a phase transition phenomenon from these results.

\begin{figure}[H]
\centering
\begin{tabular}{cc}
\begin{tikzpicture}
\begin{axis}[
width      = 0.48\textwidth,
height     = 0.4\textwidth,
xlabel = $\log(n)$,
ylabel = $\log(\text{error})$,
ymin = -2.1,
ymax = 0.1,
]
\addplot+[
domain=-10:10, 
color = black,
only marks,
mark size=1pt]
table[]
{loglog-KRR.dat};
\addplot[ultra thick, draw=red, mark=none,domain={2.2:6.3}] {-x/3+0.09896047};
\end{axis}
\end{tikzpicture}

\begin{tikzpicture}
\begin{axis}[
width      = 0.48\textwidth,
height     = 0.4\textwidth,
xlabel = $\log(n)$,
ylabel = $\log(\text{error})$,
xmin = 2.2,
xmax = 4.8,
ymin = -0.52,
ymax = 0.02,
ytick      = {-0.5, -0.4, -0.3, -0.2, -0.1, 0},
yticklabels= {-0.5, -0.4, -0.3, -0.2, -0.1, 0}
]
\addplot+[
domain=-10:10, 
color = black,
only marks,
mark size=1pt]
table[]
{slope-KRR.dat};
\addplot[ultra thick, draw=red, mark=none,domain={2:5}] {-1/3};
\end{axis}
\end{tikzpicture}
\end{tabular} \vspace{-.75\baselineskip}
\caption{Log-log plots for the plug-in KRR estimator. Left panel: Scatterplot of $(\log(n_i), \log(\error_i))$. Right panel: Slopes from rolling linear regression with moving windows of 40 observations. The reference lines in red are $y = -x/3 + \text{constant}$ (left panel) and $y = -1/3$ (right panel), both representing the minimax rate.}
\label{fig:loglog}
\end{figure}

We repeat the same experiment for local polynomial regression with degree $p=2$, where the Gaussian kernel is used and the bandwidth is selected via cross validation. The results are shown in Figure~\ref{fig:loglog2}. It can be seen that compared with the plug-in KRR estimator, local polynomial has larger errors across different sample sizes, and the error decreases at a rate slower than the optimal rate. This is not surprising as local polynomial regression lacks adaptivity to derivative orders, and we acknowledge that its performance might be improved had the tuning parameter been chosen that is better suited for the first derivative of the regression function.

\begin{figure}[H]
\centering
\begin{tabular}{cc}
\begin{tikzpicture}
\begin{axis}[
width      = 0.48\textwidth,
height     = 0.4\textwidth,
xlabel = $\log(n)$,
ylabel = $\log(\text{error})$,
ymin = -2.1,
ymax = 0.1,
]
\addplot+[
domain=-10:10, 
color = black,
only marks,
mark size=1pt]
table[]
{loglog-lp.dat};
\addplot[ultra thick, draw=red, mark=none,domain={2.2:6.3}] {-x/3+0.7704405};
\end{axis}
\end{tikzpicture}

\begin{tikzpicture}
\begin{axis}[
width      = 0.48\textwidth,
height     = 0.4\textwidth,
xlabel = $\log(n)$,
ylabel = $\log(\text{error})$,
xmin = 2.2,
xmax = 4.8,
ymin = -0.52,
ymax = 0.02,
ytick      = {-0.5, -0.4, -0.3, -0.2, -0.1, 0},
yticklabels= {-0.5, -0.4, -0.3, -0.2, -0.1, 0}
]
\addplot+[
domain=-10:10, 
color = black,
only marks,
mark size=1pt]
table[]
{slope-lp.dat};
\addplot[ultra thick, draw=red, mark=none,domain={2:5}] {-1/3};
\end{axis}
\end{tikzpicture}
\end{tabular} \vspace{-.75\baselineskip}
\caption{Log-log plots for local polynomial regression. Left panel: Scatterplot of $(\log(n_i), \log(\error_i))$. Right panel: Slopes from rolling linear regression with moving windows of 40 observations. The reference lines in red are $y = -x/3 + \text{constant}$ (left panel) and $y = -1/3$ (right panel), both representing the minimax rate.}
\label{fig:loglog2}
\end{figure}

\section{Discussion}

In this paper, we propose a plug-in kernel ridge regression estimator for estimating mixed-partial derivatives of a nonparametric regression function. The proposed estimator is analytically given and applicable for multi-dimensional support and sub-Gaussian error, enabling fast computation, broad practicability, and convenient inference. We study non-asymptotic behaviors, $L_\infty$ convergence rates, and minimax optimality of the proposed estimator. Our analysis shows that the proposed method automatically adapts to the order of derivatives to be estimated, leading to easy tuning in practice. Simulations confirm the established minimax optimality and suggest favorable performance of the proposed estimator compared to existing methods under both random and fixed designs. 

The present article is based on the commonly used iid error assumption with sub-Gaussian distributions; extension to heterogeneous or dependent error is beyond our scope but is an interesting future topic. In addition, while our theory including Theorems \ref{thm:equivalent.bound}, \ref{thm:mixed.partial}, \ref{thm:dev.deterministic.bound}(b), and \ref{thm:emb} accommodates multivariate functions and does not require the regression function $f_0$ to reside within the RKHS, the minimax optimality in the considered special examples is established for univariate functions and functions in the RKHS only. Future work could expand upon our results to consider multivariate and high-dimensional functions; challenges in this area include defining a practically useful function space (possibly with dimension-dependent smoothness levels) for interesting derivative estimations, as well as designing an appropriate kernel and parameter tuning methods that ensure rate optimality. Finally, we have focused on kernel ridge regression estimators in terms of plug-in properties for derivatives, and it is interesting to consider other related algorithms and loss functions, for example, spectral filtering based on the work of \cite{lin2020optimal} and self-concordant losses based on the work of \cite{marteau2019beyond}.

\section*{Acknowledgements}

We thank Ding-Xuan Zhou for helpful discussions, and WenWu Wang for providing R code to implement LowLSR in the simulation section. This research is partly supported by the grant DMS-2015569 and DMS/NIGMS-2153704 from the National Science Foundation.

\newpage

\appendix

\section{Proofs}
This section contains proofs of all results. We shall make use of the following Lemma~\ref{lem:K.vs.sampleK} repeatedly in the sequel, which provides an error bound for $L_{K, X} - L_K$ under the $\|\cdot\|_{\bbH}$ norm. The proof of Lemma~\ref{lem:K.vs.sampleK} mainly relies on the McDiarmid inequality and its Bernstein form, which can be found in~\cite{smale2005shannon}.
\begin{lemma}[Lemma 3 in~\cite{smale2005shannon}]
\label{lem:K.vs.sampleK}
For any Mercer kernel $K$, bounded $f \in \Ltwo$ and $0 < \delta < 1$, with probability at least $1 - \delta$, there holds 
\begin{align}
\|L_{K, X}(f) - L_K(f) \|_{\bbH} &=	\Bigg \| \frac{1}{n} \sum_{i = 1}^n f(x_i) K_{x_i} - L_K f \Bigg\|_{\bbH}\\
&\qquad\qquad\qquad \leq \frac{4 \kappa \|f\|_{\infty} }{3n} \log(1/\delta) + \frac{\kappa \|f\|_{2} }{\sqrt{n}} (1 + \sqrt{8 \log(1/\delta)}). 
\end{align}
\end{lemma}

\subsection{Proofs in Section 3.2}
\begin{proof}[Proof of Lemma~\ref{lem:RKHS.derivative}]
The proof can be found in Corollary 4.36 in \cite{steinwart2008support} or Theorem 4.7 in~\cite{Ferreira2012}.
\end{proof}

\begin{proof}[Proof of Theorem~\ref{thm:equivalent.bound}]
Letting $\Delta f=\hat{f}_n-f_\lambda$, we have
\begin{align}
L_{\tilde{K},X}(\Delta f)-L_{\tilde{K}}(\Delta f)&=L_{\tilde{K},X}(\hat{f}_n)-L_{\tilde{K},X}(f_\lambda)-L_{\tilde{K}}(\hat{f}_n)+L_{\tilde{K}}(f_\lambda).
\end{align}
Noting that $L_{\tilde{K},X}(\by-\hat{f}_n)=\hat{f}_n-L_{\tilde{K}}(\hat{f}_n)$ and $L_{\tilde{K}}(f_0)=f_\lambda$, the preceding display becomes
\begin{equation}
L_{\tilde{K},X}\by-\hat{f}_n-L_{\tilde{K},X}(f_\lambda)+L_{\tilde{K}}(f_\lambda)=L_{\tilde{K},X}(\by-f_\lambda)-\Delta f-L_{\tilde{K}}(f_0-f_\lambda).
\end{equation}
Consequently,
\begin{align}
\|\Delta f\|_{\tilde{\bbH}}&\leq \|L_{\tilde{K},X}(\Delta f)-L_{\tilde{K}}(\Delta f)\|_{\tilde{\bbH}}+\|L_{\tilde{K},X}(\by-f_\lambda)-L_{\tilde{K}}(f_0-f_\lambda)\|_{\tilde{\bbH}}\\
\label{eq:equivalent}&\leq \|L_{\tilde{K},X}(\Delta f)-L_{\tilde{K}}(\Delta f)\|_{\tilde{\bbH}}+\|L_{\tilde{K},X}(f_0-f_\lambda)-L_{\tilde{K}}(f_0-f_\lambda)\|_{\tilde{\bbH}}+\|L_{\tilde{K},X}\bw\|_{\tilde{\bbH}},
\end{align}
where $\bm{w} = \bm{y} - f_0(X)$ follows a multivariate Gaussian distribution with zero mean and variance $\sigma^2 \bm{I}_n$. Let $\Omega=[\tilde{K}(X_i,X_j)]_{i,j=1}^n$, which implies that $\|L_{\tilde{K},X}\bw\|_{\tilde{\bbH}}^2=n^{-2}\bw^T\Omega\bw$. Note that
\begin{equation}
\tr(\Omega) \leq \sum_{i=1}^{n}\tilde{K}(X_i,X_i)\leq n\tilde{\kappa}^2_\lambda\quad \text{and} \quad \tr(\Omega^2)=\sum_{i,j=1}^{n} \tilde{K}(X_i,X_j)^2\leq n^2\tilde{\kappa}_\lambda^4.
\end{equation}
According to the Hanson-Wright inequality \citep{rudelson2013hanson}, we have with probability at least $ 1 - 2e^{-ct^2}$ that
\begin{equation}
\bm{w}^T \Omega \bm{w} \leq \sigma^2\tr(\Omega)+2\sigma^2\sqrt{\tr(\Omega^2)} (t + t^2) \leq 2 \sigma^2 n \tilde{\kappa}^2_\lambda (t + 1)^2, 
\end{equation}
for any $t > 0$ and $c>0$ that does not depend on $K$ or $n$. Therefore, with probability $ 1 - \delta$, there holds 
\begin{equation}
\|L_{\tilde{K},X}\bw\|_{\tilde{\bbH}} \leq \frac{\sqrt{2}\tilde{\kappa}_\lambda \sigma}{\sqrt{n}} \left(1 + \sqrt{2c^{-1}\log(1/\delta)}\right). 
\end{equation} 
Applying Lemma~\ref{lem:K.vs.sampleK} with $\tilde{K}$ twice separately to $\Delta f$ and $f_0-f_\lambda$, with probability at least $1-3\delta$, we have
\begin{align}
\|\Delta f\|_{\tilde{\bbH}}\leq&\  \frac{4 \tilde{\kappa}_\lambda (\|\Delta f\|_{\infty}+\|f_\lambda-f_0\|_{\infty}) }{3n} \log(1/\delta) + \frac{\tilde{\kappa}_\lambda (\|\Delta f\|_{2}+\|f_\lambda-f_0\|_2) }{\sqrt{n}} \left(1 + \sqrt{8 \log(1/\delta)}\right)\\
&\ + \frac{\sqrt{2}\tilde{\kappa}_\lambda \sigma}{\sqrt{n}} \left(1 + \sqrt{2c^{-1}\log(1/\delta)}\right).
\end{align}
Note that $\|f\|_2\leq \|f\|_\infty$ for any $f\in\Ltwo$. Consider any $\delta \in (0, 1/3)$ such that $\log(1/\delta) > \log3 > 1$. Then the upper bound in the preceding inequality becomes
\begin{align}
\frac{\tilde{\kappa}_\lambda \sqrt{\log(1/\delta)}}{\sqrt{n}} \left(4 + \frac{4 \tilde{\kappa}_\lambda\sqrt{\log(1/\delta)}}{3 \sqrt{n}} \right)(\|\Delta f\|_\infty+\|f_\lambda-f_0\|_\infty)+\frac{C_1\tilde{\kappa}_\lambda \sigma\sqrt{\log(1/\delta)}}{\sqrt{n}} ,
\end{align}
where $C_1>0$ is a universal constant that does not depend on $K$ or $n$. Therefore, with probability at least $1-\delta$ for any $\delta\in(0,1)$, we have
\begin{align}
\|\Delta f\|_{\tilde{\bbH}}&\leq\frac{\tilde{\kappa}_\lambda \sqrt{\log(3/\delta)}}{\sqrt{n}} \left(4 + \frac{4 \tilde{\kappa}_\lambda\sqrt{\log(3/\delta)}}{3 \sqrt{n}} \right)(\|\Delta f\|_\infty+\|f_\lambda-f_0\|_\infty)+\frac{C_1\tilde{\kappa}_\lambda \sigma\sqrt{\log(3/\delta)}}{\sqrt{n}}.
\end{align}
By Lemma~\ref{lem:RKHS.derivative} we obtain that with probability at least $1-\delta$,
\begin{align}
\|\Delta f\|_{\tilde{\bbH}} & \leq \frac{\tilde{\kappa}_\lambda \sqrt{\log (3/\delta)}}{\sqrt{n}} \left(4 + \frac{4 \tilde{\kappa}_\lambda\sqrt{\log (3/\delta)}}{3 \sqrt{n}} \right)(\tilde{\kappa}_\lambda\|\Delta f\|_{\tilde{\bbH}}+\|f_\lambda-f_0\|_\infty)+\frac{C_1\tilde{\kappa}_\lambda\sigma\sqrt{\log (3/\delta)}}{\sqrt{n}}\\
\label{eq:thm2.intermidiate}& = C(n, \tilde{\kappa}_\lambda) \|\Delta f\|_{\tilde{\bbH}} +  \tilde{\kappa}_\lambda^{-1} C(n, \tilde{\kappa}_\lambda) \|f_\lambda-f_0\|_\infty + \frac{C_1\tilde{\kappa}_\lambda\sigma\sqrt{\log(3/\delta)}}{\sqrt{n}},
\end{align}
where
\begin{equation}
C(n,\tilde{\kappa}_\lambda)=\frac{\tilde{\kappa}^2_\lambda \sqrt{\log (3/\delta)}}{\sqrt{n}} \left(4 + \frac{4 \tilde{\kappa}_\lambda\sqrt{\log (3/\delta)}}{3 \sqrt{n}} \right).
\end{equation}
The proof is completed by applying Lemma~\ref{lem:RKHS.derivative} and the triangle inequality.

\end{proof}

\subsection{Proofs in Section 3.3}

\begin{proof}[Proof of Lemma~\ref{lem:H.derivative}]
This lemma is a variant of Lemma~\ref{lem:RKHS.derivative} but uses the $\|\cdot\|_\bbH$ instead of $\|\cdot\|_{\tilde{\bbH}}$ norm. The arguments used in proving Lemma~\ref{lem:RKHS.derivative} go verbatim.
\end{proof}

\begin{lemma}\label{thm:noiseless.RKHS}
For any bounded $f \in \Ltwo$, let
\begin{align}
E(K, X, f) & := (L_{K, X} + \lambda I)^{-1} L_{K, X} f - (L_K + \lambda I)^{-1} L_K  f  \\
\label{eq:def.E}& = K(\cdot, X)[K(X, X) + n \lambda \bm{I}_n]^{-1}f(X) - (L_K + \lambda I)^{-1} L_K  f, 
\end{align}
For any $\delta \in (0, 1)$, it holds with probability at least $1 - \delta$  that
\begin{equation}\label{eq:noiseless.delta}
\|E(K, X, f) \|_\bbH\leq\frac{\kappa \|f\|_\infty \sqrt{\log(3/\delta)}}{\sqrt{n}\lambda}  \left(10 + \frac{4 \kappa\sqrt{\log(3/\delta)}}{3 \sqrt{n\lambda}} \right).
\end{equation}
\end{lemma}

\begin{proof}[Proof of Lemma~\ref{thm:noiseless.RKHS}]
We introduce an intermediate quantity $(L_{K, X} + \lambda I)^{-1}L_{K}f$ and decompose $E(K, X, f)  = (\tilde{L}_{K,X}f - (L_{K, X} + \lambda I)^{-1}L_{K}f) + ((L_{K, X} + \lambda I)^{-1}L_{K}f - \tilde{L}_K f)$. We will calculate error bounds for both terms by applying Lemma~\ref{lem:K.vs.sampleK} twice. 
First we have
\begin{align}
& \ \| \tilde{L}_{K,X} f - (L_{K, X} + \lambda I)^{-1} L_{K}f  \|_{\bbH}\\
=&\ \|(L_{K, X} + \lambda I)^{-1} (L_{K, X} f - L_K f)\|_\bbH\\
\leq&\ \frac{1}{\lambda} \left \|L_{K, X} f - L_K f \right\|_{\bbH},
\end{align}
where the last inequality is due to \eqref{eq:dummy1} in the main paper. Applying Lemma~\ref{lem:K.vs.sampleK}, then with probability at least $1 - \delta$, we have  
\begin{equation}\label{eq:bound2} 
\|\tilde{L}_{K,X} f - (L_{K, X} + \lambda I)^{-1} L_{K}f  \|_{\bbH} \leq \frac{4 \kappa \|f\|_{\infty}}{3n\lambda} \log(1/\delta) + \frac{\kappa \|f\|_{2}}{\sqrt{n}\lambda} (1 + \sqrt{8 \log(1/\delta)}). 
\end{equation}
On the other hand, we have
\begin{align}
&\ \| (L_{K, X}  + \lambda I)^{-1}L_K f-\tilde{L}_{K}f\|_\bbH \\
=&\  \|(L_{K, X}  + \lambda I)^{-1} (L_K + \lambda I) \tilde{L}_{K}f - (L_{K, X}  + \lambda I)^{-1}(L_{K, X}  + \lambda I)\tilde{L}_{K}f \|_\bbH\\
=&\ \|(L_{K, X}  + \lambda I)^{-1}(L_K \tilde{L}_{K}f - L_{K, X}  \tilde{L}_{K}f )\|_\bbH\\
\leq&\ \frac{1}{\lambda} \|L_K \tilde{L}_{K}f - L_{K, X} \tilde{L}_{K}f\|_{\bbH}.
\end{align}
Applying Lemma~\ref{lem:K.vs.sampleK} to $\tilde{L}_{K}f$ gives 
\begin{equation}
\| (L_{K, X}  + \lambda I)^{-1}L_K f-\tilde{L}_{K}f\|_{\bbH} \leq \frac{4 \kappa \|\tilde{L}_{K}f\|_{\infty}}{3n\lambda} \log(1/\delta) + \frac{\kappa \|\tilde{L}_{K}f\|_{2}}{\sqrt{n}\lambda} (1 + \sqrt{8 \log(1/\delta)}). 
\end{equation}
Letting $f = 0$ in \eqref{eq:KRR.population} gives 
\begin{equation}
\|f_{\lambda} - f_{0}\|_{2}^2 + \lambda \|f_{\lambda}\|^2_{\bbH} \leq \|f_{0}\|_{2}^2, 
\end{equation}
which yields 
\begin{equation} \label{eq:f.lambda.loose.bound}
\|f_{\lambda}\|_{2} \leq \sqrt{2} \|f_0\|_{2} \quad \text{and}  \quad \|f_{\lambda}\|_{\bbH} \leq \lambda^{-1/2} \|f_0\|_{2}. 
\end{equation}
By Lemma~\ref{lem:H.derivative} we have $\|\tilde{L}_{K}f\|_{\infty} \leq \kappa \|\tilde{L}_{K}f\|_{\bbH}$. This together with \eqref{eq:f.lambda.loose.bound} gives
\begin{equation}\label{eq:bound3} 
\|(L_{K, X}  + \lambda I)^{-1}L_K f-\tilde{L}_{K}f \|_{\bbH} \leq \frac{4 \kappa^2 \|f\|_{2}/\sqrt{\lambda}}{3n\lambda} \log(1/\delta) + \frac{\sqrt{2}\kappa \|f\|_{2}}{\sqrt{n}\lambda} (1 + \sqrt{8 \log(1/\delta)}). 
\end{equation}
Again consider any $\delta \in (0, 1/3)$ such that $\log(1/\delta) > \log3 > 1$. Then, the two bounds in equations~\eqref{eq:bound2} and~\eqref{eq:bound3} become
\begin{equation}
\frac{4\kappa \|f\|_\infty}{3 {n} \lambda} \log(1/\delta) + \frac{4 \kappa \|f\|_\infty}{\sqrt{n} \lambda} \sqrt{\log(1/\delta)}, \quad \frac{4\kappa^2  \|f\|_\infty/\sqrt{\lambda}}{3 n \lambda} \log(1/\delta) + \frac{6\kappa  \|f\|_\infty}{\sqrt{n} \lambda} \sqrt{\log(1/\delta)}, 
\end{equation}
respectively. Consequently, with probability at least $1 - 2 \delta > 1-3\delta$, we have 
\begin{align}
\|E(K, X, f)\|_{\bbH} =\|\tilde{L}_{K,X}f-\tilde{L}_K f\|_\bbH& \leq \frac{\kappa \|f\|_\infty}{\sqrt{n}\lambda} \left( 10 \sqrt{\log(1/\delta)} + \frac{4}{3 \sqrt{n}} \log(1/\delta) +  \frac{4 \kappa}{3 \sqrt{n\lambda}} \log(1/\delta)\right) \\
& \leq  \frac{\kappa  \|f\|_\infty\sqrt{\log(1/\delta)}}{\sqrt{n}\lambda} \left(10 + \frac{4 \kappa \sqrt{\log(1/\delta)}}{3 \sqrt{n\lambda}} \right). 
\end{align}
Therefore, with probability at least $1-\delta$ for any $\delta\in(0,1)$, we have
\begin{align}
\|E(K, X, f)\|_{\bbH} &\leq\frac{\kappa  \|f\|_\infty\sqrt{\log(3/\delta)}}{\sqrt{n}\lambda} \left(10 + \frac{4 \kappa\sqrt{\log(3/\delta)}}{3 \sqrt{n\lambda}} \right).
\end{align}
\end{proof}

\begin{proof}[Proof of Theorem~\ref{thm:mixed.partial}]
Substituting $f = f_0$ into $E(K, X, f)$ defined in~\eqref{eq:def.E} yields $E(K, X, f_0)=f_{X,\lambda}- f_\lambda$. By Lemma~\ref{thm:noiseless.RKHS}, we have with probability at least $1-\delta$ that
\begin{equation}\label{eq:general.bound.1}
\|f_{X,\lambda}-f_\lambda\|_\bbH\leq\frac{\kappa \|f_0\|_\infty \sqrt{\log(3/\delta)}}{\sqrt{n}\lambda}  \left(10 + \frac{4 \kappa\sqrt{\log(3/\delta)}}{3 \sqrt{n\lambda}} \right).
\end{equation}
Note that
\begin{equation}
\hat{f}_n - f_{X, \lambda} = K(\cdot, X)[K(X, X) + n \lambda \bI_n]^{-1} \bm{w} = K(\cdot, X)[K(X, X)/n + \lambda \bI_n]^{-1} \bm{w} /n, 
\end{equation}
where $\bm{w} = \bm{y} - f_0(X)$ follows a multivariate Gaussian distribution with zero mean and variance $\sigma^2 \bm{I}_n$. Thus,
\begin{align}
\| \hat{f}_n - f_{X, \lambda} \|_{\bbH} ^2 &= \frac{1}{n^2} \bm{w}^T [K(X, X)/n + \lambda \bI_n]^{-1} K(X, X) [K(X, X)/n +  \lambda \bI_n]^{-1} \bm{w} \\
&\leq  \frac{1}{n^2} \kappa^2 \bm{w}^T \Sigma \bm{w},  
\end{align}
where $\Sigma =[K(X, X)/n +\lambda\bI_n]^{-2}$. Since $K(X, X)/n$ is non-negative definite, all eigenvalues of $K(X, X)/n + \lambda \bI_n$ are bounded below by $\lambda$, which leads to
\begin{equation}
\tr(\Sigma) \leq n \lambda^{-2}\quad \text{and} \quad \tr(\Sigma^2) \leq n^2 \lambda^{-4}. 
\end{equation}
According to the Hanson-Wright inequality \citep{rudelson2013hanson}, we have with probability at least $ 1 - 2e^{-ct^2}$ that
\begin{equation}
\bm{w}^T \Sigma \bm{w} \leq \sigma^2\tr(\Sigma)+ 2\sigma^2 \sqrt{\tr(\Sigma^2)} (t + t^2) \leq 2 \sigma^2 n \lambda^{-2} (t + 1)^2, 
\end{equation}
for any $t > 0$. Therefore, with probability $ 1 - \delta$, there holds 
\begin{equation}\label{eq:general.bound.2}
\| \hat{f}_n - f_{X, \lambda} \|_{\bbH}  \leq \frac{\sqrt{2} \kappa \sigma}{\sqrt{n} \lambda} \left(1 + \sqrt{2c^{-1}\log(1/\delta)}\right)\leq\frac{C_2\kappa \sigma \sqrt{\log(1/\delta)}}{\sqrt{n} \lambda},
\end{equation}
where we consider any $\delta \in (0, 1/3)$ such that $\log(1/\delta) > \log3 > 1$ and $C_2>0$ is a universal constant that does not depend on $K$ or $n$. Combining \eqref{eq:general.bound.1} and \eqref{eq:general.bound.2}, it holds that with probability at least $1 - 2\delta > 1-3\delta$,
\begin{equation}
\|\hat{f}_n-f_\lambda\|_\bbH \leq\frac{\kappa \|f_0\|_\infty \sqrt{\log(3/\delta)}}{\sqrt{n}\lambda}  \left(10 + \frac{4 \kappa\sqrt{\log(3/\delta)}}{3 \sqrt{n\lambda}} \right)+\frac{C_2 \kappa \sigma \sqrt{\log(1/\delta)}}{\sqrt{n} \lambda}.
\end{equation}
Hence, for any $\delta\in(0,1)$, it holds with probability at least $1-\delta$ that
\begin{align}
\|\hat{f}_n-f_\lambda\|_\bbH &\leq\frac{\kappa \|f_0\|_\infty \sqrt{\log(9/\delta)}}{\sqrt{n}\lambda}  \left(10 + \frac{4 \kappa\sqrt{\log(9/\delta)}}{3 \sqrt{n\lambda}} \right)+\frac{C_2 \kappa \sigma \sqrt{\log(3/\delta)}}{\sqrt{n} \lambda}.
\end{align}
The proof is completed by applying Lemma~\ref{lem:H.derivative} and the triangle inequality.
\end{proof}

\begin{proof}[Proof of Corollary~\ref{cor:sharp}]
We first simplify \bounda{}. With $\delta=n^{-10}$, we have $$C(n,\tilde{\kappa}_\lambda)=\frac{\tilde{\kappa}^2_\lambda \sqrt{10\log(3n)}}{\sqrt{n}} \left(4 + \frac{4 \tilde{\kappa}_\lambda\sqrt{10\log(3n)}}{3 \sqrt{n}} \right).$$ The condition $\tilde{\kappa}_\lambda^2 = o(\sqrt{n/\log n})$ yields that $C(n,\tilde{\kappa}_\lambda) = o(1)$ and further $\tilde{\kappa}_{\bm\beta,\lambda}\tilde{\kappa}_\lambda^{-1}C(n,\tilde{\kappa}_\lambda) \lesssim \tilde{\kappa}_{\bm\beta,\lambda}\tilde{\kappa}_\lambda \sqrt{\log n/n}$. Noting that $\|f_\lambda - f_0\|_\infty = o(1)$, the second term in \bounda{} is bounded by the third term. Hence, \bounda{} becomes
\begin{equation} \label{eq:bounda.simple}
\|\partial^{\bm\beta}f_\lambda-\partial^{\bm\beta}f_0\|_\infty +\frac{C_1'\tilde{\kappa}_{\bm\beta,\lambda}\tilde{\kappa}_\lambda\sigma\sqrt{10\log(3n)}}{\sqrt{n}}.
\end{equation}

With $\delta=n^{-10}$, \boundb{} becomes
\begin{equation}
\|\partial^{\bm\beta}f_\lambda - \partial^{\bm\beta} f_0\|_{\infty}+\frac{\kappa_{\bm\beta}\kappa \|f_0\|_\infty \sqrt{10\log(9n)}}{\sqrt{n}\lambda} \left(10 + \frac{4\kappa\sqrt{10\log(9n)}}{3 \sqrt{n\lambda}} \right) +\frac{C_2\kappa_{\bm\beta}\kappa \sigma \sqrt{10\log(3n)}}{\sqrt{n} \lambda}.
\end{equation}
Comparing the preceding display with \boundb{}, we can see that if $\tilde\kappa_{\bm\beta,\lambda}\tilde\kappa_\lambda = o(\lambda^{-1})$, \bounda{} is asymptotically less than \boundb{}.

\end{proof}

\begin{proof}[Proof of Theorem~\ref{thm:dev.deterministic.bound}]
We first prove (a). Rewrite $f_0$ as $f_0=L_K^{r}g$ for some $g=L_K^{-r}f_0\in \Ltwo$ and thus $f_i=\mu_i^{r}g_i$. Representing the function $g$ by $g=\sum_{i=1}^{\infty}  g_i\psi_i$, we have
\begin{equation}
f_\lambda-f_0=-\sum_{i=1}^{\infty}  \frac{\lambda}{\mu_i+\lambda}\mu_i^{r} g_i\psi_i.
\end{equation}
When $\frac{1}{2} <r\leq 1$, we have
\begin{align}
\|f_\lambda-f_0\|_\bbH^2&=\sum_{i=1}^{\infty} \left(\frac{\lambda}{\mu_i+\lambda}\mu_i^{r} g_i\right)^2/\mu_i\\
&=\lambda^{2r-1}\sum_{i=1}^{\infty} \left(\frac{\lambda}{\mu_i+\lambda}\right)^{3-2r}\left(\frac{\mu_i}{\mu_i+\lambda}\right)^{2r-1}g_i^2\\
&\leq \lambda^{2r-1}\|L_K^{-r}f_0\|^2_{2}.
\end{align}
The proof is completed by applying Lemma~\ref{lem:H.derivative}.

For (b), we have $\partial^{\bm\beta} f_\lambda-\partial^{\bm\beta} f_0=-\sum_{i=1}^{\infty}  \frac{\lambda}{\mu_i+\lambda}\mu_i^{r} g_i\partial^{\bm\beta} \psi_i$, where $\{\psi_i\}_{i=1}^{\infty}$ is the Fourier basis, \textit{i.e.}, $\psi_1(\bx)=1,\psi_{2i}(\bx)=\cos(2\pi \bm I_i\cdot \bx),\psi_{2i+1}=\sin(2\pi \bm I_i\cdot \bx)$; here $\bm I_i\in \mathbb{N}_0^d$ are ordered multi-indexes. It follows that
\begin{align}
\|\partial^{\bm\beta}f_\lambda-\partial^{\bm\beta}f_0\|_\infty &\leq \s \frac{\lambda}{\mu_i+\lambda}\mu_i^r |g_i||\partial^{\bm\beta}\psi_i|\\
&= \lambda^r\s \left(\frac{\lambda}{\mu_i+\lambda}\right)^{1-r}\left(\frac{\mu_i}{\mu_i+\lambda}\right)^r|g_i||\partial^{\bm\beta}\psi_i|
\leq\lambda^r\s |g_i||\partial^{\bm\beta}\psi_i|.
\end{align}
Since $g\in C^p(\mX)$, the Fourier coefficients satisfy $|g_i|\lesssim \binom{i+d}{d-1}i^{-p}\lesssim i^{d-p-1}$. Moreover, $|\partial^{\bm\beta} \psi_i|\lesssim i(i-1)\cdots(i-|\bm\beta|+1)\lesssim i^{|\bm\beta|}$. Therefore,
\begin{equation}
\|\partial^{\bm\beta}f_\lambda-\partial^{\bm\beta}f_0\|_\infty\leq C_3\lambda^r \s i^{d-p-1+|\bm\beta|}
=C_3\lambda^r\zeta(p-d+1+|\bm\beta|).
\end{equation}
\end{proof}

\begin{proof}[Proof of Theorem~\ref{thm:emb}]
We write $f_0=L_K^r g$ for some $g=L_K^{-r}f_0\in \Ltwo$ and thus $f_i=\mu_i^r g_i$. Representing the function $g$ by $g=\sum_{i=1}^{\infty}  g_i\psi_i$, then we have
\begin{equation}
f_\lambda-f_0=-\sum_{i=1}^{\infty}  \frac{\lambda}{\mu_i+\lambda}\mu_i^r g_i\psi_i.
\end{equation}

We have $\partial^{\bm\beta} f_\lambda-\partial^{\bm\beta} f_0=-\sum_{i=1}^{\infty}  \frac{\lambda}{\mu_i+\lambda}\mu_i^{r} g_i\partial^{\bm\beta} \psi_i$, where $\{\psi_i\}_{i=1}^{\infty}$ is the Fourier basis. Define $\{g_i^*\}_{i=1}^\infty$ such that $g_i\partial^{\bm\beta}\psi_i = g_i^* i^{\bm\beta} \psi_i$ and let $g^* = \s g_i^*\psi_i \in \Ltwo$. 

According to Lemma 10 in \cite{fischer2020sobolev}, Assumption~\ref{ass:emb} implies that the eigenvalues decay with a polynomial upper bound of order $1/q$, \textit{i.e.}, there exists a constant $C_4>0$ such that for all $i\in\mathbb{N}$,
\begin{equation}\label{eq:decay}
\mu_i \leq C_4 i^{-1/q},
\end{equation}
which implies that $i^{|\bm\beta|} \leq C_4^{q|\bm\beta|}\mu_i^{-q|\bm\beta|} = C_4 \mu_i^{-q|\bm\beta|}$. Thus,
\begin{align}
\|\partial^{\bm\beta}f_\lambda-\partial^{\bm\beta}f_0\|_\infty &\leq \left\|\s \frac{\lambda}{\mu_i+\lambda}\mu_i^r g_i\partial^{\bm\beta}\psi_i \right\|_\infty\\
&\leq \lambda \sup_{i\geq 1}\frac{\mu_i^{r-q/2}i^{|\bm\beta|}}{\mu_i+\lambda}  \left\|\s \mu_i^{q/2} g_i^* \psi_i \right\|_\infty\\
&\leq C_4\lambda \sup_{i\geq 1}\frac{\mu_i^{r-q/2-q|\bm\beta|}}{\mu_i+\lambda}  \left\|L_K^{q/2}g^*\right\|_\infty\\
&\leq AC_4\lambda^{r-q/2-q|\bm\beta|}\|g^*\|_2,
\end{align}
where the last inequality follows from Lemma 25 in \cite{fischer2020sobolev} and Assumption~\ref{ass:emb}.
\end{proof}

\subsection{Proofs in Section 4}
\begin{proof}[Proof of Lemma~\ref{lem:differentiability.matern}]
When $\alpha > m + 1/2$, we have
\begin{equation}
\partial^{m,m} K_\alpha(x,x') = \s \mu_i \psi_i^{(m)}(x) \psi_i^{(m)}(x') \lesssim \s i^{-2\alpha} i^{2m} < \infty.
\end{equation}
Thus, $K_\alpha \in C^{2m}([0,1] \times [0,1])$.

Recall the definition of $\tilde{\kappa}_{m,\lambda}^2$ in \eqref{eq:high.order.kappa}. It follows that for any $m\in\mathbb{N}_0$,
\begin{align}
\tilde{\kappa}_{\alpha,m,\lambda}^2&=\sup_{x\in[0,1]}\s\frac{\mu_i}{\lambda+\mu_i}\psi^{(m)}_i(x)^2\\
&\lesssim \sum_{i=1}^{\infty}\frac{i^{2m}}{1+\lambda i^{2\alpha}}\leq
\int_0^{\infty}\frac{(x+1)^{2m}dx}{1+\lambda x^{2\alpha}}\asymp  \lambda^{-\frac{2m+1}{2\alpha}},
\end{align}
where the last step holds for $\alpha>m+\frac{1}{2}$. On the other hand, we have
\begin{align}
\tilde{\kappa}_{\alpha,m,\lambda}^2&\gtrsim\s\left[\frac{(2i)^{2m}}{1+\lambda (2i)^{2\alpha}}\cos(2\pi ix)^2+\frac{(2i+1)^{2m}}{1+\lambda (2i+1)^{2\alpha}}\sin(2\pi ix)^2\right]\\
&\geq \s \min\left\{\frac{(2i)^{2m}}{1+\lambda (2i)^{2\alpha}},\frac{(2i+1)^{2m}}{1+\lambda (2i+1)^{2\alpha}}\right\}\\
&\geq \frac{1}{2} \sum_{i=1}^{\infty}\frac{i^{2m}}{1+\lambda i^{2\alpha}}\asymp  \lambda^{-\frac{2m+1}{2\alpha}},
\end{align}
where we also need $\alpha>m+\frac{1}{2}$. The differentiability of $\tilde{K}_\alpha$ directly follows from the boundedness of $\tilde{\kappa}^2_{\alpha,m,\lambda}$ for any fixed $\lambda$.
\end{proof}

\begin{proof}[Proof of Lemma~\ref{lem:matern.RKHS.derivative}]
In view of Lemma~\ref{lem:RKHS.derivative} and Lemma~\ref{lem:differentiability.matern}, we can see that $f\in C^m[0,1]$ for any $f \in \tilde{\bbH}_\alpha$. This is also true for $f \in \bbH_\alpha$ since $\bbH_\alpha$ and $\tilde{\bbH}_\alpha$ contain the same functions.

Now we prove the norm inequality. Let $f=\s f_i\psi_i$, where $\{\psi_i\}_{i=1}^\infty$ is the Fourier basis. Then,  $\|f^{(m)}\|_2^2\asymp \s (f_i i^m)^2$ for any $m\in\mathbb{N}_0$. It is equivalent to showing that
\begin{equation}
\tilde{\kappa}_{\alpha,\lambda}^2 \cdot \s f_i^2i^{2m} \leq C_m\tilde{\kappa}_{\alpha,m,\lambda}^2 \cdot \s f_i^2\frac{\lambda+\mu_i}{\mu_i},
\end{equation}
for some $C_m>0$. Hence, it suffices to show that for any $i\in\mathbb{N}$,
\begin{equation}
\tilde{\kappa}_{\alpha,\lambda}^2 \cdot f_i^2i^{2m}\leq C_m \tilde{\kappa}_{\alpha,m,\lambda}^2 \cdot f_i^2\frac{\lambda+\mu_i}{\mu_i}.
\end{equation}
In view of Lemma~\ref{lem:differentiability.matern}, we have $\tilde{\kappa}_{\alpha,m,\lambda}^2 \asymp \lambda^{-\frac{2m+1}{2\alpha}}$, which also leads to $\tilde{\kappa}_{\alpha,\lambda}^2 \asymp \lambda^{-\frac{1}{2\alpha}}$ when taking $m = 0$. Since $\mu_i\asymp i^{-2\alpha}$, it is sufficient to show that
\begin{equation}\label{eq:matern.ineq}
\lambda^{\frac{m}{\alpha}} i^{2m}\leq C_m (1+\lambda i^{2\alpha}),
\end{equation}
for some constant $C_m>0$. The above equation trivially holds for $C_m = 1$ if $\lambda^{\frac{m}{\alpha}} i^{2m}\leq 1$. If $\lambda^{\frac{m}{\alpha}} i^{2m}\geq 1$, then $\lambda^{\frac{m}{\alpha}} i^{2m}\leq (\lambda^{\frac{m}{\alpha}} i^{2m})^{\frac{\alpha}{m}}=\lambda i^{2\alpha}$ since $m<\alpha$. Taking $C_m = 1$ completes the proof.
\end{proof}

\begin{proof}[Proof of Lemma~\ref{lem:matern.deriv.deterministic}]
Let $f_0=\s f_i\psi_i$. Then,
\begin{equation}
f_\lambda-f_0=-\s\frac{\lambda}{\lambda+\mu_i}f_i\psi_i.
\end{equation}
Note that
\begin{align}
\|f_\lambda-f_0\|^2_{\tilde{\bbH}_{\alpha}} &= \s \left(\frac{\lambda}{\lambda+\mu_i} f_i\right)^2 \bigg/ \frac{\mu_i}{\lambda+\mu_i} = \lambda \s \frac{\lambda}{\lambda+\mu_i} \frac{f_i^2}{\mu_i}\\
&\lesssim \lambda \s i^{2\alpha} f_i^2 \leq \lambda \left(\s i^{\alpha} |f_i| \right)^2 \lesssim \lambda.
\end{align}
Therefore, for any $f_0 \in H^\alpha[0,1]$ or $f_0 \in S^\alpha[0,1]$, we have $\|f_\lambda-f_0\|_{\tilde{\bbH}_{\alpha}} \lesssim \lambda^{\frac{1}{2}}$.
\end{proof}

\begin{proof}[Proof of Theorem~\ref{thm:minimax.diff.matern}]
Applying Lemma~\ref{lem:matern.RKHS.derivative} to \eqref{eq:thm2.intermidiate} and taking $\delta=n^{-10}$ yields with $\PP_0^{(n)}$-probability at least $1-n^{-10}$ that
\begin{align}
\|\hat{f}_n^{(m)}-f_\lambda^{(m)}\|_2& \leq \tilde{\kappa}_{\alpha,\lambda}^{-1}\tilde{\kappa}_{\alpha,m,\lambda} \|\hat{f}_n-f_\lambda\|_{\tilde{\bbH}_\alpha}\\
&\leq\frac{\tilde{\kappa}_{\alpha,m,\lambda}\tilde{\kappa}_{\alpha,\lambda}^{-2}C(n,\tilde{\kappa}_{\alpha,\lambda})}{1-C(n,\tilde{\kappa}_{\alpha,\lambda})}\|f_\lambda-f_0\|_\infty+\frac{1}{1-C(n,\tilde{\kappa}_{\alpha,\lambda})}\frac{C_1\tilde{\kappa}_{\alpha,m,\lambda}\sigma\sqrt{10\log(3n)}}{\sqrt{n}}.
\end{align}
By choosing $\lambda$ such that $\tilde{\kappa}_{\alpha,\lambda}^2=o(\sqrt{n/\log n})$, we have $\tilde{\kappa}_{\alpha,\lambda}^{-2}C(n,\tilde{\kappa}_{\alpha,\lambda})\asymp \sqrt{\log n/n}$ and $C(n,\tilde{\kappa}_{\alpha,\lambda})\leq 1/2$ for sufficiently large $n$. We arrive at
\begin{equation}
\|\hat{f}_n^{(m)}-f_\lambda^{(m)}\|_2 \lesssim  \frac{2\tilde{\kappa}_{\alpha,m,\lambda}\sqrt{\log n}}{\sqrt{n}}\|f_\lambda-f_0\|_\infty+\frac{2C_1\tilde{\kappa}_{\alpha,m,\lambda}\sigma\sqrt{10\log(3n)}}{\sqrt{n}}\lesssim \tilde{\kappa}_{\alpha,m,\lambda} \sqrt{\frac{\log n}{n}},
\end{equation}
given that $\|f_\lambda - f_0\|_\infty = o(1)$. Hence,
\begin{equation}\label{eq:l2.bound}
\|\hat{f}_n^{(m)}-f_0^{(m)}\|_2 \lesssim \|f_\lambda^{(m)}-f_0^{(m)}\|_2+ \tilde{\kappa}_{\alpha,m,\lambda} \sqrt{\frac{\log n}{n}}.
\end{equation}
In view of Lemma~\ref{lem:RKHS.derivative} and Lemma~\ref{lem:matern.deriv.deterministic}, we can see that for any $f_0\in H^\alpha[0,1]$ or $f_0\in S^\alpha[0,1]$, $\|f_\lambda - f_0\|_\infty \leq \tilde{\kappa}_{\alpha,\lambda}\|f_\lambda - f_0\|_{\tilde{\bbH}_\alpha} \lesssim \lambda^{\frac{1}{2}-\frac{1}{4\alpha}} = o(1)$. Invoking Lemma~\ref{lem:matern.RKHS.derivative}, we have
\begin{equation}
\|f_\lambda^{(m)}-f_0^{(m)}\|_2 \lesssim \tilde{\kappa}_{\alpha,\lambda}^{-1} \tilde{\kappa}_{\alpha,m,\lambda} \|f_\lambda-f_0\|_{\tilde{\bbH}_{\alpha}} \lesssim \lambda^{\frac{1}{4\alpha}} \lambda^{-\frac{2m+1}{4\alpha}} \lambda^{\frac{1}{2}} = \lambda^{\frac{1}{2}-\frac{m}{2\alpha}}.
\end{equation}
It follows from \eqref{eq:l2.bound} that
\begin{equation}
\|\hat{f}_n^{(m)}-f_0^{(m)}\|_2 \lesssim \lambda^{\frac{1}{2}-\frac{m}{2\alpha}}+\lambda^{-\frac{2m+1}{4\alpha}}\sqrt{\frac{\log n}{n}}.
\end{equation}
The upper bound in the preceding display is minimized when $\lambda\asymp ({\log n}/{n})^{\frac{2\alpha}{2\alpha+1}}$, which satisfies $\tilde{\kappa}_{\alpha,\lambda}^2\asymp(n/\log n)^{\frac{1}{2\alpha+1}}=o(\sqrt{n/\log n})$. The optimal rate is derived by substituting $\lambda$. This completes the proof.
\end{proof}

\section{Additional simulation results under fixed design}

We conduct a Monte Carlo study in the fixed design setting. We consider the same functions and sample size $n = 500$ as in Section 6.1 of the main paper and choose fixed design points $X_i=i/500$ for $i=1,\ldots, 500$. We run 100 repetitions and evaluate each estimator according to RMSE as in Section 6.1.

Figure~\ref{fig:boxplot-fix}--\ref{fig:boxplot-fix-3} display the boxplots of RMSEs for estimating up to the third derivative in the fixed design setting. It can be seen that LowLSR performs slightly better than in the random design setting, especially when estimating the derivatives of $f_{02}$. Still, the two KRR methods compare favorably to the benchmarks, and the Mat\'ern kernel gives the best median RMSEs for almost all cases, with one exception in the left panel of Figure~\ref{fig:boxplot-fix} where it ties with the proposed estimator with Sobolev kernel. This is consistent with our observations made in the random design setting. 

Figure~\ref{fig:function-fix} shows the result from one random run in our Monte Carlo study for estimating the first derivatives in the fixed design setting. We observe similar trends as in the random design setting in Figure 4. For example, the estimation of LowLSR appears undesirably wiggly when estimating $f'_{01}$. These results suggest that the proposed plug-in KRR estimator continues to work well in the fixed design setting.

\newpage

\begin{figure}[H]
\centering
\begin{tabular}{cc}
\begin{tikzpicture}
\begin{axis}[
width      = 0.5\textwidth,
height     = 0.35\textwidth,
boxplot/draw direction=y,
boxplot/every box/.style={fill=gray!50},
boxplot/box extend=0.8,
xticklabel style = {align=center, font=\small, rotate=60},
xtick={1,2,3,4,5},
xticklabels={
Sobolev, Mat\'ern , locpol, pspline, LowLSR},
ymin = -0.1,
ymax = 1.6,
ytick={0,0.5,1,1.5},
yticklabels={0,0.5,1,1.5},
]
\addplot[boxplot] table[y index=0] {fix-11.dat};
\addplot[boxplot] table[y index=1] {fix-11.dat};
\addplot[boxplot] table[y index=2] {fix-11.dat};
\addplot[boxplot] table[y index=3] {fix-11.dat};
\addplot[boxplot] table[y index=4] {fix-11.dat};
\end{axis}
\end{tikzpicture}
\begin{tikzpicture}
\begin{axis}[
width      = 0.5\textwidth,
height     = 0.35\textwidth,
boxplot/draw direction=y,
boxplot/every box/.style={fill=gray!50},
boxplot/box extend=0.8,
xticklabel style = {align=center, font=\small, rotate=60},
xtick={1,2,3,4,5},
xticklabels={
Sobolev, Mat\'ern , locpol, pspline, LowLSR},
ymin = -0.1,
ymax = 2.6,
ytick={0,0.5,1,1.5,2,2.5},
yticklabels={0,0.5,1,1.5,2,2.5},
]
\addplot[boxplot] table[y index=0] {fix-21.dat};
\addplot[boxplot] table[y index=1] {fix-21.dat};
\addplot[boxplot] table[y index=2] {fix-21.dat};
\addplot[boxplot] table[y index=3] {fix-21.dat};
\addplot[boxplot] table[y index=4] {fix-21.dat};
\end{axis}
\end{tikzpicture}
\end{tabular} \vspace{-.75\baselineskip}
\caption{Boxplots of RMSEs for estimating $f'_{01}$ (left) and $f'_{02}$ (right) in fixed design setting.}
\label{fig:boxplot-fix}
\vspace{\baselineskip}
\begin{tabular}{cc}
\begin{tikzpicture}
\begin{axis}[
width      = 0.5\textwidth,
height     = 0.35\textwidth,
boxplot/draw direction=y,
boxplot/every box/.style={fill=gray!50},
boxplot/box extend=0.8,
xticklabel style = {align=center, font=\small, rotate=60},
xtick={1,2,3,4,5},
xticklabels={
Sobolev, Mat\'ern , locpol, pspline, LowLSR},
ymin = -1,
ymax = 21,
ytick={0,5,10,15,20},
yticklabels={0,5,10,15,20},
]
\addplot[boxplot] table[y index=0] {fix-12.dat};
\addplot[boxplot] table[y index=1] {fix-12.dat};
\addplot[boxplot] table[y index=2] {fix-12.dat};
\addplot[boxplot] table[y index=3] {fix-12.dat};
\addplot[boxplot] table[y index=4] {fix-12.dat};
\end{axis}
\end{tikzpicture}
\begin{tikzpicture}
\begin{axis}[
width      = 0.5\textwidth,
height     = 0.35\textwidth,
boxplot/draw direction=y,
boxplot/every box/.style={fill=gray!50},
boxplot/box extend=0.8,
xticklabel style = {align=center, font=\small, rotate=60},
xtick={1,2,3,4,5},
xticklabels={
Sobolev, Mat\'ern , locpol, pspline, LowLSR},
ymin = -2,
ymax = 52,
ytick={0,10,20,30,40,50},
yticklabels={0,10,20,30,40,50},
]
\addplot[boxplot] table[y index=0] {fix-22.dat};
\addplot[boxplot] table[y index=1] {fix-22.dat};
\addplot[boxplot] table[y index=2] {fix-22.dat};
\addplot[boxplot] table[y index=3] {fix-22.dat};
\addplot[boxplot] table[y index=4] {fix-22.dat};
\end{axis}
\end{tikzpicture}
\end{tabular} \vspace{-.75\baselineskip}
\caption{Boxplots of RMSEs for estimating $f''_{01}$ (left) and $f''_{02}$ (right) in fixed design setting.}
\label{fig:boxplot-fix-2}
\vspace{\baselineskip}
\begin{tabular}{cc}
\begin{tikzpicture}
\begin{axis}[
width      = 0.5\textwidth,
height     = 0.35\textwidth,
boxplot/draw direction=y,
boxplot/every box/.style={fill=gray!50},
boxplot/box extend=0.8,
xticklabel style = {align=center, font=\small, rotate=60},
xtick={1,2,3,4},
xticklabels={
Sobolev, Mat\'ern , locpol, pspline},
ymin = -10,
ymax = 310,
ytick={0,50,100,150,200,250,300},
yticklabels={0,50,100,150,200,250,300},
]
\addplot[boxplot] table[y index=0] {fix-13.dat};
\addplot[boxplot] table[y index=1] {fix-13.dat};
\addplot[boxplot] table[y index=2] {fix-13.dat};
\addplot[boxplot] table[y index=3] {fix-13.dat};
\end{axis}
\end{tikzpicture}
\begin{tikzpicture}
\begin{axis}[
width      = 0.5\textwidth,
height     = 0.35\textwidth,
boxplot/draw direction=y,
boxplot/every box/.style={fill=gray!50},
boxplot/box extend=0.8,
xticklabel style = {align=center, font=\small, rotate=60},
xtick={1,2,3,4},
xticklabels={
Sobolev, Mat\'ern , locpol, pspline},
ymin = -100,
ymax = 2100,
ytick={0,500,1000,1500,2000},
yticklabels={0,500,1000,1500,2000},
]
\addplot[boxplot] table[y index=0] {fix-23.dat};
\addplot[boxplot] table[y index=1] {fix-23.dat};
\addplot[boxplot] table[y index=2] {fix-23.dat};
\addplot[boxplot] table[y index=3] {fix-23.dat};
\end{axis}
\end{tikzpicture}
\end{tabular} \vspace{-.75\baselineskip}
\caption{Boxplots of RMSEs for estimating $f'''_{01}$ (left) and $f'''_{02}$ (right) in fixed design setting. LowLSR is not applicable to estimate the third derivative.}
\label{fig:boxplot-fix-3}
\end{figure}

\begin{figure}[H]
\centering
\begin{tabular}{cc}
\begin{tikzpicture}
\begin{axis}[
width      = 0.48\textwidth,
height     = 0.4\textwidth,
xlabel = $x$,
ylabel = $f'_{01}(x) \text{, } \hat f'_{01}(x)$,
xmin = -0.05,
xmax = 1.05,
ymin = -3.2,
ymax = 2.2,
ytick      = {-3, -2, -1, 0, 1, 2},
yticklabels= {-3, -2, -1, 0, 1, 2},
]
\addplot[mark=none,  black,   ultra thick] table[x index = 0, y index = 1]{fix-1.dat};
\addplot[mark=none,  green, very thick, dashed] table[x index = 0, y index = 2]{fix-1.dat};
\addplot[mark=none,  red, very thick, dashed] table[x index = 0, y index = 3]{fix-1.dat};
\addplot[mark=none,  blue, very thick, dash pattern={on 10pt off 4pt}] table[x index = 0, y index = 4]{fix-1.dat};
\addplot[mark=none,  yellow, very thick, dash pattern={on 10pt off 4pt}] table[x index = 0, y index = 5]{fix-1.dat};
\addplot[mark=none,  gray, very thick, dash pattern={on 10pt off 4pt}] table[x index = 0, y index = 1]{fix-1-supp.dat};
\end{axis}
\end{tikzpicture}

\begin{tikzpicture}
\begin{axis}[
width      = 0.48\textwidth,
height     = 0.4\textwidth,
xlabel = $x$,
ylabel = $f'_{02}(x) \text{, } \hat f'_{02}(x)$,
xmin = -0.05,
xmax = 1.05,
ymin = -16,
ymax = 16,
ytick      = {-15, -10, -5, 0, 5, 10, 15},
yticklabels= {-15, -10, -5, 0, 5, 10, 15},
]
\addplot[mark=none,  black,   ultra thick] table[x index = 0, y index = 1]{fix-2.dat};
\addplot[mark=none,  green, very thick, dashed] table[x index = 0, y index = 2]{fix-2.dat};
\addplot[mark=none,  red, very thick, dashed] table[x index = 0, y index = 3]{fix-2.dat};
\addplot[mark=none,  blue, very thick, dash pattern={on 10pt off 4pt}] table[x index = 0, y index = 4]{fix-2.dat};
\addplot[mark=none,  yellow, very thick, dash pattern={on 10pt off 4pt}] table[x index = 0, y index = 5]{fix-2.dat};
\addplot[mark=none,  gray, very thick, dash pattern={on 10pt off 4pt}] table[x index = 0, y index = 1]{fix-2-supp.dat};
\end{axis}
\end{tikzpicture}
\end{tabular} \vspace{-.75\baselineskip}
\caption{One random run in the Monte Carlo study for estimating $f'_{01}$ (left) and $f'_{02}$ (right) under fixed design: true derivative (full line), KRR with Sobolev kernel (green dashed line), Mat\'ern kernel (red dashed line), locpol (blue long dash), spline (yellow long dash) and LowLSR (grey long dash).}
\label{fig:function-fix}
\end{figure}

\bibliographystyle{apalike}
\bibliography{KRR_v4}

\end{document}